\documentclass[american]{llncs}
\usepackage{todonotes}
\usepackage{babel}
\usetikzlibrary{shapes,shapes.geometric,backgrounds,calc,fit,positioning}

\usepackage[utf8]{inputenc}
\usepackage{graphicx}

\usepackage[ruled,vlined]{algorithm2e}

\usepackage{amsmath,amssymb}
\usepackage{mathtools}
\usepackage{commath}
\usepackage{faktor} 
\usepackage{ntheorem}
\usepackage{nicefrac}
\usepackage{enumerate}
\usepackage{paralist}

\usepackage[hidelinks]{hyperref}
\usepackage[all]{hypcap}

\usepackage[noabbrev]{cleveref}
\let\cref\Cref

\usepackage[backend=biber,style=numeric-comp,giveninits=true,url=false]{biblatex}
\usepackage{booktabs}
\usepackage{arydshln}

\title{Quantifying the Conceptual Error in Dimensionality Reduction}

\date{tba 2020}

\institute{University of Kassel, Germany}

\usepackage{fca, newdrawline}

\crefalias{problem}{prob}

\newcommand{\Scon}{\mathbb{S}}
\newcommand{\Tcon}{\mathbb{T}}

\newcommand{\Sh}{\mathfrak{S}}
\newcommand{\SH}{\underline{\mathfrak{S}}}

\DeclareMathOperator{\id}{id}

\newcommand*{\logeq}{\ratio\Leftrightarrow}

\DeclareMathOperator{\Ext}{Ext}
\DeclareMathOperator{\Int}{Int}

\DeclareMathOperator{\app}{\mid}

\usepackage{etoolbox,refcount}
\usepackage{multicol}
\newcounter{countitems}
\newcounter{nextitemizecount}
\newcommand{\setupcountitems}{%
  \stepcounter{nextitemizecount}%
  \setcounter{countitems}{0}%
  \preto\item{\stepcounter{countitems}}%
}
\makeatletter
\newcommand{\computecountitems}{%
  \edef\@currentlabel{\number\c@countitems}%
  \label{countitems@\number\numexpr\value{nextitemizecount}-1\relax}%
}
\newcommand{\nextitemizecount}{%
  \getrefnumber{countitems@\number\c@nextitemizecount}%
}
\newcommand{\previtemizecount}{%
  \getrefnumber{countitems@\number\numexpr\value{nextitemizecount}-1\relax}%
}
\makeatother    
    {\end{itemize}%
    \unskip\computecountitems\ifnumcomp{\previtemizecount}{>}{3}{\end{multicols}}{}}
\newcommand\blfootnote[1]{%
  \begingroup
  \renewcommand\thefootnote{}\footnote{#1}%
  \addtocounter{footnote}{-1}%
  \endgroup
}
\newcommand{\putorcid}[1]{\makebox[0pt][r]{\raisebox{-1ex}[0pt][0pt]{\textcolor{red}{\tiny  #1}}}}

\begin{document}
\parindent0pt

\author{Tom Hanika\inst{1,2}\putorcid{0000-0002-4918-6374}%
\and Johannes Hirth\inst{1,2}\putorcid{0000-0001-9034-0321}
}
\institute{%
 Knowledge \& Data Engineering Group,
 University of Kassel, Germany\\[0.5ex]
 \and
 Interdisciplinary Research Center for Information System Design\\
 University of Kassel, Germany\\[0.5ex]
 \email{tom.hanika@cs.uni-kassel.de, hirth@cs.uni-kassel.de}
}

\maketitle
\blfootnote{Authors are given in alphabetical order.
  No priority in authorship is implied.}

\begin{abstract}
  Dimension reduction of data sets is a standard problem in the realm
  of machine learning and knowledge reasoning. They affect patterns in
  and dependencies on data dimensions and ultimately influence any
  decision-making processes.  Therefore, a wide variety of reduction
  procedures are in use, each pursuing different objectives. A so far
  not considered criterion is the conceptual continuity of the
  reduction mapping, i.e., the preservation of the conceptual
  structure with respect to the original data set. Based on the notion
  scale-measure from formal concept analysis we present in this work
  a) the theoretical foundations to detect and quantify conceptual errors
  in data scalings; b) an experimental investigation of our
  approach on eleven data sets that were respectively treated with a
  variant of non-negative matrix factorization.
\end{abstract}

\section{Introduction}
The analysis of large and complex data is presently a challenge for
many data driven research fields. This is especially true when using
sophisticated analysis and learning methods, since their computational
complexity usually grows at least superlinearly with the problem
size. One aspect of largeness and complexity is the explicit data
dimension, e.g., number of features, of a data set. Therefore, a
variety of methods have been developed to reduce exactly this data
dimension to a computable size, such as \emph{principal component
  analysis}, \emph{singular value
  decomposition}, or \emph{factor
  analysis}~\cite{descreteBasis}. What all these methods have in common is that they are
based on the principle of data scaling~\cite{navimeasure}.

A particularly challenging task is to apply \emph{Boolean factor
  analysis} (BFA)~\cite{bmf}, as the distinct feature values are
restricted to either 0 (false) or 1 (true). For example, given the
binary data set matrix $K$, the application of a BFA yields two binary
data matrices $S,H$ of lower dimension, such that $S\cdot H$
approximates $K$ with respect to a previously selected norm
$\|\cdot\|$. The factor $S$ can be considered as a lower dimensional
representation of $K$, i.e., a scaling of $K$. The connection between
the scaling features of $S$ and the original data features of $K$ is
represented by $H$. The quality of an approximation, and therewith the
quality of a scale $S$, is usually scored through the Frobenius norm of
$K-SH$, or other functions~\cite{NMFApplication2,MeasuresNMF}, such as
the \emph{Residual Sum of Squares} or, in the binary setting, the
\emph{Hemming distance}. These scoring functions give a good
impression of the extent to which the linear operator $K$ is
approximated by $S\cdot H$, yet, they are incapable to detect the
deviation of internal incidence structure of $S$ with respect to $K$,
which we want to call \emph{conceptual scaling error}.

A well defined formalism from Formal Concept Analysis
(FCA)~\cite{fca-book} to analyze the resulting inconsistencies in
(binary) data scaling is \emph{scale-measures}~\cite{cmeasure,
  scaling}. In this work, we build up on this notion and introduce a
comprehensive framework for quantification of quantifying the
conceptual errors of scales. The so introduced mathematical tools are
capable of determining how many conceptual errors arise from a
particular scaling $S$ and pinpoint which concepts are falsely
introduced or lost. For this we overcome the potential exponential
computational demands of computing complete conceptual structures by
employing previous results on the scale-measures decision
problem~\cite{navimeasure}. We motivate our results with accompanying
examples and support our results with an experiment on eleven data
sets.

\section{Scales and Data}

\subsubsection*{FCA Recap}
In the field of Formal concept analysis (FCA)~\cite{Wille1982,
  fca-book} the task for data scaling, and in particular feature
scaling, is considered a fundamental step for data analysis. Hence,
data scaling is part of the foundations of FCA~\cite{fca-book} and it
is frequently investigated within FCA~\cite{cmeasure, scaling}.

The basic data structure of FCA is the \emph{formal context}, see our
running example $\context_{W}$~\cref{fig:bj1} (top). That is a triple
$(G,M,I)$ with non-empty and finite\footnote{This restriction is in
  general not necessary. Since data sets can be considered finite
  throughout this work and since this assumption allows a clearer
  representation of the results within this work, it was made.} set
$G$ (called \emph{objects}), finite set $M$ (called \emph{attributes})
and a binary relation $I \subseteq G \times M$ (called
\emph{incidence}). The tuple $(g,m)$ indicates ``$g$ has attribute $m$
in $\context$'', when $(g,m)\in I$.  Any context $\Scon=(H,N,J)$ with
$H\subseteq G, N\subseteq M$ and $J=I\cap (H \times N)$ we
call \emph{induced sub-context} of $\context$, and denote this
relation by $\leq$, i.e., $\Scon \leq \context$.

The incidence relation $I\subseteq G\times M$ gives rise to
\emph{the} natural Galois connection between $P(G)$ and $P(M)$, which
is the pair\footnote{Both operators are traditionally denoted by the
  same symbol $\cdot'$} of operators $\cdot'\colon\mathcal{P}(G)\to
\mathcal{P}(M),$ $A\mapsto A'=\{m\in M \mid \forall a \in A\colon(a,m)\in
I\}$, and $\cdot'\colon\mathcal{P}(M)\to\mathcal{P}(G),$ $B\mapsto
B'=\{g\in G\mid \forall b\in B\colon(g,b)\in I\}$, each called
\emph{derivation}. Using these operator, a \emph{formal concept} is a
pair $(A,B)\in \mathcal{P}(G)\times \mathcal{P}(M)$ with $A'=B$ and $A
= B'$, where $A$ and $B$ are called \emph{extent} and \emph{intent},
respectively.

The two possible compositions of the derivation operators lead to two
\emph{closure operators}
$\cdot''\colon\mathcal{P}(G)\to\mathcal{P}(G)$ and
$\cdot''\colon\mathcal{P}(M)\to\mathcal{P}(M)$, and in turn to two
\emph{closure spaces} $\Ext(\context)\coloneqq (G,'')$ and
$\Int(\context)\coloneqq(M,'')$. Both closure systems are represented
in the \emph{(concept) lattice}
$\BV(\context)=(\mathcal{B}(\context),\leq)$, where the set of all
formal concepts is denoted by $\mathcal{B}(\context)\coloneqq
\{(A,B)\in\mathcal{P}(G)\times\mathcal{P}(M)\mid A'=B\wedge B'=A\}$
the order relation is $(A,B)\leq (C,D):\logeq A\subseteq C$. An
example drawing of such a lattice is depicted in~~\cref{fig:bj1} (bottom).

\begin{figure}[t]
  \label{bjice}
  \centering
    \scalebox{0.55}{
      \hspace{-1.4cm}
      \begin{cxt}
        \cxtName{$\context_{W}$}
        \att{\shortstack{has limbs\\ (L) }}
        \att{\shortstack{breast feeds\\ (BF) }}
        \att{\shortstack{needs\\ chlorophyll (Ch)}}
        \att{\shortstack{needs water\\ to live (W)}}
        \att{\shortstack{lives on\\ land (LL)}}
        \att{\shortstack{lives in\\ water (LW)}}
        \att{\shortstack{can move\\ \ (M)}}
        \att{\shortstack{monocotyledon\\ (MC) }}
        \att{\shortstack{dicotyledon\\ (DC) }}
        \obj{xx.xx.x..}{\shortstack{\ dog\\ \ }}
        \obj{...x.xx..}{\shortstack{\ fish\\ \ leech }}
        \obj{..xxx..x.}{\shortstack{\ corn\\ \ }}
        \obj{x..x.xx..}{\shortstack{\ bream\\ \ }}
        \obj{..xx.x.x.}{\shortstack{\ water\\ \ weeds}}
        \obj{..xxx...x}{\shortstack{\ bean\\ \ }}
        \obj{x..xxxx..}{\shortstack{\ frog\\ \ }}
        \obj{..xxxx.x.}{\shortstack{\ reed\\ \ }}
      \end{cxt}}  

    \scalebox{0.4}{\colorlet{mivertexcolor}{black!80}
\colorlet{jivertexcolor}{black!80}
\colorlet{vertexcolor}{black!80}
\colorlet{bordercolor}{black!80}
\colorlet{linecolor}{gray}
\tikzset{vertexbase/.style={semithick, shape=circle, inner sep=1pt, outer sep=1pt, draw=bordercolor},%
  vertex/.style={vertexbase, fill=vertexcolor!45},%
  mivertex/.style={vertexbase, fill=mivertexcolor!45},%
  jivertex/.style={vertexbase, fill=jivertexcolor!45},%
  divertex/.style={vertexbase, top color=mivertexcolor!45, bottom color=jivertexcolor!45},%
  conn/.style={-, thick, color=linecolor}%
}
\begin{tikzpicture}[scale=0.35,font=\huge]
  \begin{scope} 
    \begin{scope} 
      \foreach \nodename/\nodetype/\xpos/\ypos in {%
        0/vertex/0.0/0.0,
        1/jivertex/-6.0/10.0,
        2/jivertex/6.0/10.0,
        3/divertex/-12.0/14.0,
        4/jivertex/-6.0/14.0,
        5/jivertex/5.0/15.0,
        6/jivertex/9.0/15.0,
        7/divertex/13.0/15.0,
        8/vertex/-11.0/17.0,
        9/jivertex/-5.0/17.0,
        10/vertex/0.0/18.0,
        11/mivertex/7.0/19.0,
        12/vertex/11.0/19.0,
        13/mivertex/-11.0/21.0,
        14/mivertex/-1.0/23.0,
        15/mivertex/-10.0/24.0,
        16/mivertex/10.0/24.0,
        17/mivertex/1.0/31.0,
        18/vertex/0.0/36.0
      } \node[\nodetype] (\nodename) at (\xpos, \ypos) {};
    \end{scope}
    \begin{scope} 
      \path (2) edge[conn] (6);
      \path (1) edge[conn] (4);
      \path (12) edge[conn] (17);
      \path (6) edge[conn] (11);
      \path (14) edge[conn] (18);
      \path (0) edge[conn] (2);
      \path (9) edge[conn] (14);
      \path (2) edge[conn] (5);
      \path (10) edge[conn] (17);
      \path (7) edge[conn] (12);
      \path (0) edge[conn] (1);
      \path (6) edge[conn] (12);
      \path (15) edge[conn] (18);
      \path (9) edge[conn] (15);
      \path (5) edge[conn] (14);
      \path (17) edge[conn] (18);
      \path (1) edge[conn] (10);
      \path (0) edge[conn] (7);
      \path (8) edge[conn] (17);
      \path (13) edge[conn] (15);
      \path (4) edge[conn] (13);
      \path (0) edge[conn] (3);
      \path (8) edge[conn] (13);
      \path (4) edge[conn] (9);
      \path (5) edge[conn] (11);
      \path (12) edge[conn] (16);
      \path (2) edge[conn] (10);
      \path (3) edge[conn] (8);
      \path (16) edge[conn] (18);
      \path (11) edge[conn] (16);
      \path (10) edge[conn] (14);
      \path (1) edge[conn] (8);
    \end{scope}
    \begin{scope} 
      \foreach \nodename/\labelpos/\labelopts/\labelcontent in {%
        1/below//{frog},
        2/below//{reed},
        3/below//{dog},
        3/above//{BF},
        4/below//{bream},
        5/below//{ww},
        6/below//{corn},
        7/below//{bean},
        7/above//{DC},
        9/below//{fish leech},
        11/above//{MC},
        13/above//{L},
        14/above//{LW},
        15/above//{M},
        16/above//{Ch},
        17/above//{LL},
        18/above//{W}
      } \coordinate[label={[\labelopts]\labelpos:{\labelcontent}}](c) at (\nodename);
    \end{scope}
  \end{scope}
\end{tikzpicture}}
  \caption{This formal context of the \emph{Living Beings and Water}
    data set (top) and its concept
    lattice (bottom).}
  \label{fig:bj1}
\end{figure}
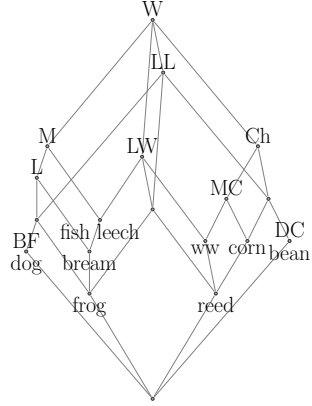

\subsection{Scales-Measures}\label{sec:motivate}
The basis for quantifying the conceptual error of a data scaling is
the continuity of the scaling map with respect to the resulting
closure spaces. We say a map map $f:G_1\to G_2$ is \emph{continuous}
with respect to the closure spaces $(G_1,c_1)$ and $(G_2,c_2)$ if and
only if $\text{for all}\ A\in\mathcal{P}(G_2) \text{ we have }
c_1(f^{-1}(A))\subseteq f^{-1}(c_2(A))$. That is, a map is continuous
iff the preimage of closed sets is closed. Within FCA this notion is
captured by the following definition.

\begin{definition}[Scale-Measure (cf. Definition 91, \cite{fca-book})]
\label{def:sm}
Let $\context = (G,M,I)$ and
$\mathbb{S}=(G_{\mathbb{S}},M_{\mathbb{S}},I_{\mathbb{S}})$ be a
formal contexts. The map $\sigma :G \rightarrow G_{\mathbb{S}}$ is
called an \emph{$\mathbb{S}$-measure of $\context$ into the scale
  $\mathbb{S}$} iff the preimage
$\sigma^{-1}(A)\coloneqq \{g\in G\mid \sigma(g)\in A\}$
of every extent $A\in \Ext(\Scon)$ is an extent of $\context$.
\end{definition}

The (scaling-)map $\sigma$ can be understood as an interpretation of
the objects from $\context$ using the attribute (features) of
$\Scon$. Hence, we denote
$\sigma^{-1}[\Ext(\Scon)]\coloneqq\bigcup_{A\in\Ext(\Scon)}\sigma^{-1}(A)$
as the set of extents that are \emph{reflected} by the \emph{scale
  context} $\Scon$.

\begin{figure}[t]
  \label{bjicemeasure}
\hspace{-1cm}\scalebox{0.78}{
      \begin{cxt}
        \cxtName{$\Scon$}
        \att{\shortstack{\textbf{W} \\ \, }}
        \att{\shortstack{\textbf{LW} \\ \, }}
        \att{\shortstack{\textbf{plants} \\ $\coloneqq$ C}} 
        \att{\shortstack{\textbf{animals} \\ $\coloneqq$  M}} 
        \att{\shortstack{\textbf{land plants}\\ $\coloneqq $LL$\wedge $ plant}} 
        \att{\shortstack{\textbf{water plants} \\ $\coloneqq$ LW $\wedge$ plant}} 
        \att{\shortstack{\textbf{land animal}\\ $\coloneqq$ LL $\wedge$ animal}} 
        \att{\shortstack{\textbf{water animal} \\ $\coloneqq$ LW $\wedge$ animal}} 
        \att{\shortstack{\textbf{mammal} \\ $\coloneqq$ animal $\wedge$ BF}} 
        \obj{x..x..x.x}{\shortstack{\ dog\\ \ }}
        \obj{xx.x...x.}{\shortstack{\ fish\\ \ leech }}
        \obj{x.x.x....}{\shortstack{\ corn\\ \ }}
        \obj{xx.x...x.}{\shortstack{\ bream\\ \ }}
        \obj{xxx..x...}{\shortstack{\ water\\ \ weeds}}
        \obj{x.x.x....}{\shortstack{\ bean\\ \ }}
        \obj{xx.x..xx.}{\shortstack{\ frog\\ \ }}
        \obj{xxx.xx...}{\shortstack{\ reed\\ \ }}
      \end{cxt}}  

  \begin{tikzpicture}
    \node at (0,0)
    {\scalebox{0.5}{\colorlet{mivertexcolor}{white}
\colorlet{jivertexcolor}{black}
\colorlet{vertexcolor}{black}
\colorlet{bordercolor}{black}
\colorlet{linecolor}{gray}
\tikzset{vertexbase/.style={semithick, shape=circle, inner sep=2pt, outer sep=0pt, draw=bordercolor,draw opacity=0.4},%
  vertex/.style={vertexbase, fill=vertexcolor!45},%
  mivertex/.style={vertexbase, fill=mivertexcolor!45},%
  jivertex/.style={vertexbase, fill=jivertexcolor!45},%
  divertex/.style={vertexbase, top color=mivertexcolor!45, bottom color=jivertexcolor!45},%
  conn/.style={-, thick, color=linecolor}%
}
\tikzstyle{n} = [text width=2.5cm,align=center]
\begin{tikzpicture}[scale=0.35,font=\Large]
  \begin{scope} 
    \begin{scope} 
      \foreach \nodename/\nodetype/\xpos/\ypos in {%
        0/vertex/0.0/0.0,
        1/vertex/-6.0/10.0,
        2/vertex/6.0/10.0,
        3/vertex/-12.0/14.0,
        4/mivertex/-6.0/14.0,
        5/vertex/5.0/15.0,
        6/mivertex/9.0/15.0,
        7/mivertex/13.0/15.0,
        8/vertex/-11.0/17.0,
        9/vertex/-5.0/17.0,
        10/mivertex/0.0/18.0,
        11/mivertex/7.0/19.0,
        12/vertex/11.0/19.0,
        13/mivertex/-11.0/21.0,
        14/vertex/-1.0/23.0,
        15/vertex/-10.0/24.0,
        16/vertex/10.0/24.0,
        17/mivertex/1.0/31.0,
        18/vertex/0.0/36.0
      } \node[\nodetype] (\nodename) at (\xpos, \ypos) {};
    \end{scope}
    \begin{scope} 
      \path (2) edge[conn,draw opacity=0.8] (6);
      \path (1) edge[conn,draw opacity=0.8] (4);
      \path (12) edge[conn,draw opacity=0.8] (17);
      \path (6) edge[conn,draw opacity=0.8] (11);
      \path (14) edge[conn,draw opacity=0.8] (18);
      \path (0) edge[conn,draw opacity=0.8] (2);
      \path (9) edge[conn,draw opacity=0.8] (14);
      \path (2) edge[conn,draw opacity=0.8] (5);
      \path (10) edge[conn,draw opacity=0.8] (17);
      \path (7) edge[conn,draw opacity=0.8] (12);
      \path (0) edge[conn,draw opacity=0.8] (1);
      \path (6) edge[conn,draw opacity=0.8] (12);
      \path (15) edge[conn,draw opacity=0.8] (18);
      \path (9) edge[conn,draw opacity=0.8] (15);
      \path (5) edge[conn,draw opacity=0.8] (14);
      \path (17) edge[conn,draw opacity=0.8] (18);
      \path (1) edge[conn,draw opacity=0.8] (10);
      \path (0) edge[conn,draw opacity=0.8] (7);
      \path (8) edge[conn,draw opacity=0.8] (17);
      \path (13) edge[conn,draw opacity=0.8] (15);
      \path (4) edge[conn,draw opacity=0.8] (13);
      \path (0) edge[conn,draw opacity=0.8] (3);
      \path (8) edge[conn,draw opacity=0.8] (13);
      \path (4) edge[conn,draw opacity=0.8] (9);
      \path (5) edge[conn,draw opacity=0.8] (11);
      \path (12) edge[conn,draw opacity=0.8] (16);
      \path (2) edge[conn,draw opacity=0.8] (10);
      \path (3) edge[conn,draw opacity=0.8] (8);
      \path (16) edge[conn,draw opacity=0.8] (18);
      \path (11) edge[conn,draw opacity=0.8] (16);
      \path (10) edge[conn,draw opacity=0.8] (14);
      \path (1) edge[conn,draw opacity=0.8] (8);
    \end{scope}
    \begin{scope} 
      \foreach \nodename/\labelpos/\labelopts/\labelcontent in {%
        1/below//{frog},
        2/below//{reed},
        3/below//{dog},
        3/above//{breast feeds},
        4/below//{bream},
        5/below//{water weeds},
        6/below//{corn},
        7/below//{bean},
        7/above//{DC},
        9/below//{fish leech},
        11/above//{MC},
        13/above//{L},
        14/above//{LW},
        15/above//{M},
        16/above//{Ch},
        17/above//{LL},
        18/above//{W}
      } \coordinate[label={[\labelopts]\labelpos:{\labelcontent}}](c) at (\nodename);
    \end{scope}
  \end{scope}
\end{tikzpicture}}};
    \node[draw opacity = 0,draw=white, text=black] at (4,0)
    {$\Rightarrow$};
    \node at (7,0)
    {\scalebox{0.7}{\colorlet{mivertexcolor}{black!80}
\colorlet{jivertexcolor}{black!80}
\colorlet{vertexcolor}{black!80}
\colorlet{bordercolor}{black!80}
\colorlet{linecolor}{gray}
\tikzset{vertexbase/.style={semithick, shape=circle, inner sep=2pt, outer sep=0pt, draw=bordercolor},%
  vertex/.style={vertexbase, fill=vertexcolor!45},%
  mivertex/.style={vertexbase, fill=mivertexcolor!45},%
  jivertex/.style={vertexbase, fill=jivertexcolor!45},%
  divertex/.style={vertexbase, top color=mivertexcolor!45, bottom color=jivertexcolor!45},%
  conn/.style={-, color=linecolor}%
}
\tikzstyle{n} = [text width=2cm,align=center]
\begin{tikzpicture}[scale=0.35,font=\large]
  \begin{scope} 
    \begin{scope} 
      \foreach \nodename/\nodetype/\xpos/\ypos in {%
        0/vertex/0.0/0.0,
        1/jivertex/-3.5/7.0,
        2/jivertex/3.50/7.0,
        4/divertex/-8.0/10.0,
        5/jivertex/-3.0/11.0,
        6/jivertex/3.5/11.0,
        7/mivertex/-7.0/13.0,
        8/mivertex/8.0/14.0,
        9/mivertex/-6.0/16.0,
        10/mivertex/0.0/16.0,
        11/mivertex/6.0/16.0,
        12/vertex/0.0/24.0
      } \node[\nodetype] (\nodename) at (\xpos, \ypos) {};
    \end{scope}
    \begin{scope} 
      \path (6) edge[conn,draw opacity=0.8] (10);
      \path (1) edge[conn,draw opacity=0.8] (5);
      \path (7) edge[conn,draw opacity=0.8] (9);
      \path (0) edge[conn,draw opacity=0.8] (2);
      \path (4) edge[conn,draw opacity=0.8] (7);
      \path (5) edge[conn,draw opacity=0.8] (9);
      \path (6) edge[conn,draw opacity=0.8] (11);
      \path (0) edge[conn,draw opacity=0.8] (1);
      \path (5) edge[conn,draw opacity=0.8] (10);
      \path (8) edge[conn,draw opacity=0.8] (11);
      \path (9) edge[conn,draw opacity=0.8] (12);
      \path (0) edge[conn,draw opacity=0.8] (4);
      \path (1) edge[conn,draw opacity=0.8] (7);
      \path (2) edge[conn,draw opacity=0.8] (6);
      \path (2) edge[conn,draw opacity=0.8] (8);
      \path (10) edge[conn,draw opacity=0.8] (12);
      \path (11) edge[conn,draw opacity=0.8] (12);
    \end{scope}
    \begin{scope} 
      \foreach \nodename/\labelpos/\labelopts/\labelcontent in {%
        12/above/n/{W},                                    
        1/below/n/{frog},
        2/below/n/{reed},
        4/below/n/{dog},
        4/above/n/{mammal}, 
        5/below/n/{fish leech, bream},
        5/above//{water animal}, 
        6/below/n/{water weeds},
        6/above//{water plant}, 
        7/above//{land animal}, 
        8/below/n/{corn, bean},
        8/above//{land plant}, 
        9/above/n/{animal}, 
        10/above/n/{LW}, 
        11/above/n/{plant} 
      } \coordinate[label={[\labelopts]\labelpos:{\labelcontent}}](c) at (\nodename);
    \end{scope}
  \end{scope}
\end{tikzpicture}}};
  \end{tikzpicture}
  \caption{A scale context (top), its concept lattice (bottom right)
    for which $\id_G$ is a scale-measure of the context in
    \cref{bjice} and the reflected extents
    $\sigma^{-1}[\Ext(\Scon)]$ (bottom left) indicated as
    in gray.}
\end{figure}
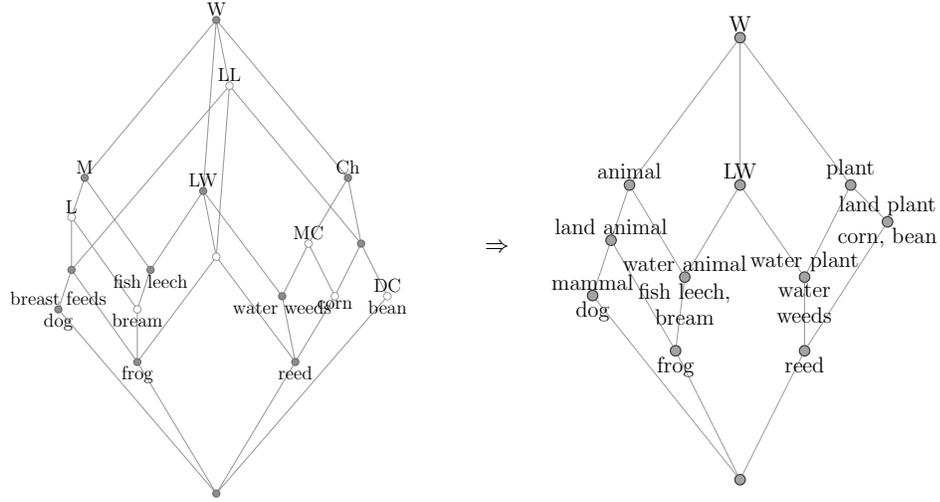

In~\cref{bjicemeasure} we depicted an example scale-context $\Scon$,
where the attributes of the scale-context are constructed by
conjunctions of attributes from $\context_{W}$, as seen
in~\cref{fig:bj1}.  This scaling is based on the original object set
$G$ and we observe that $\Scon$ reflects twelve out of
the nineteen concepts from $\mathfrak{B}(\context_{\text{W}})$.

For any two scale-measures $(\sigma,\Scon),(\psi,\Tcon)$ we
say~\cite{navimeasure} that $(\sigma,\Scon)$ is \emph{finer than}
$(\psi,\Tcon)$, iff $\psi^{-1}(\Tcon)\subseteq
\sigma^{-1}[\Ext(\Scon)]$. Dually we say then that $(\sigma,\Scon)$ is
\emph{coarser than} $(\psi,\Tcon)$. From both relation an equivalence
relation $\sim$ arises naturally.  The set of all possible
scale-measures for some $\context$, denoted by $\Sh(\context)\coloneqq
\{(\sigma, \Scon)\mid \sigma$ is a $\Scon-$measure of $\context \}$,
is therefore ordered. Furthermore, it is known~\cite{navimeasure} that
factorizing $\Sh(\context)$ by $\sim$ leads to a lattice ordered
structure $\SH(\context) = (\nicefrac{\Sh(\context)}{\sim},\leq)$,
called the \emph{scale-hierarchy of $\context$}. This hierarchy is
isomorphic to the set of all sub-closure systems of $\Ext(\context)$,
i.e. $\{Q\subseteq \Ext(\context) \mid Q \text{ is a Closure System on
}G\}$, ordered by set inclusion $\subseteq$.


Every scale-measure $(\Scon,\sigma)\in \Sh(\context)$ does allow for a
canonical representation~\cite{navimeasure}, i.e., $(\sigma,\Scon)\sim
(\id, \context_{\sigma^{-1}(\Ext(\Scon))})$. This representation,
however, very often eludes human explanation to some degree. This
issue can be remedied through a related approach called \emph{logical
  scaling}~\cite{logiscale} that is a representation of the
scale-context attributes by conjunction, disjunction, and negation of
the original attributes, formally
$M_\Scon\subseteq\mathcal{L}(M,\{\wedge,\vee,\neg\})$. Such a
representation does always exist:

\begin{proposition}[CNF of Scale-measures (cf. Proposition 23, \cite{navimeasure})] \label{lem:appconst} 
Let $\context$ be a
  context, $(\sigma,\Scon)\in \Sh(\context)$. Then the scale-measure
  $(\psi,\Tcon)\in \Sh(\context)$ given by
  \[\psi = \id_G\quad \text{ and }\quad \Tcon =
    \app\limits_{A\in\sigma^{-1}[\Ext(\Scon)]} (G,\{\phi = \wedge\
    A^{I}\},I_{\phi}) \] is equivalent to $(\sigma,\Scon)$ and is
  called \emph{conjunctive normalform of}
  $(\sigma,\Scon)$.  
\end{proposition}

\section{Conceptual Errors in Data Scaling}
Scaling and factorizations procedures are essential to almost all data
science (DS) and machine learning (ML) approaches. For example,
relational data, such as a formal context $\context$, is often scaled
to a lower (attribute-) dimensional representation $\Scon$. Such
scaling procedures, e.g., \emph{principle component analysis, latent
  semantic analysis, non-negative matrix factorization}, however, do
almost always not account for the conceptual structure of the original
data $\context$. Hence, in order to comprehend the results of DS/ML
procedures, it is crucial to investigate to what extent and which
information is lost during the scaling process. In the following we
want to introduce a first approach to quantify and treat \emph{error}
in data scalings through the notion of scale-measures. For this we
need the notion of \emph{context apposition}~\cite{fca-book}. Given
two formal contexts $\context_{1}\coloneqq(G_{1},M_{1},I_{1})$,
$\context_{2}\coloneqq(G_{2},M_{2},I_{2})$ with $G_{1}=G_{2}$ and
$M_{1}\cap M_{2}=\emptyset$, then $\context_{1}\mid
\context_{2}\coloneqq (G,M_{1}\cup M_{2},I_{1}\cup I_{2})$. If
$M_{1}\cap M_{2}\neq\emptyset$ the apposition is constructed disjoint
union of the attribute sets.

\begin{proposition}\label{prop:err-measure}
  Let $\context$, $\Scon$ be formal contexts,
  $\sigma:G_{\context}\to G_{\Scon}$ a map and let $\mathbb{A} = \context 
  \mid (G_{\context},M_\Scon,I_\sigma)$ with $I_{\sigma}= \{(g,\sigma(g))\mid
  g\in G_{\context}\}\circ I_{\Scon}$, then $(\sigma,\Scon)$ and $(\id_{G_{\context}},\context)$ are scale-measures of
  $\mathbb{A}$.
\end{proposition}
\begin{proof}
  It is to show that $\id_{G_{\context}}^{-1}[\Ext(\context)]$ and
  $\sigma^{-1}[\Ext(\Scon)]$ are subsets of $\Ext(\mathbb{A})$. The
  set $\Ext(\mathbb{A})$ is the smallest intersection closed set
  containing $\Ext(\context)$ and
  $\Ext(G_{\context},M_\Scon,I_\sigma)$. Thus
  $\id_{G_{\context}}^{-1}(\Ext(\context))\subseteq
  \Ext(\mathbb{A})$.
  Let $C\in \Ext(\Scon)$, then there exists a representation $C =
  D^{I_{\Scon}}$ with $D\subseteq M_{\Scon}$. The  derivation
  $D^{I_{\sigma}}$ is an extent in $\Ext(G_{\context},M_\Scon,I_\sigma)$ and is equal
  to $\sigma^{-1}(D^{I_{\Scon}})$, since $I_{\sigma}
  =\{(g,\sigma(g))\mid g\in G_{\context}\}\circ I_{\Scon}$. Thus,
  $D^{I_{\sigma}}= \sigma^{-1}(C)$ and therefore we find $\sigma^{-1}(C)\in
  \Ext(G_{\context},M_\Scon,I_\sigma)$. 
\end{proof}

The extent closure system of the in~\cref{prop:err-measure} constructed
context $\mathbb{A}$, is equal to the join of $\Ext(\context)$ and
$\sigma^{-1}[\Ext(\Scon)]$ in the closure system of all closure systems on
$G$ (cf. Proposition 13 \cite{navimeasure}). Hence, $\Ext(A)$ is the smallest
closure system on $G$, for which $\Ext(\context)$ and
$\sigma^{-1}[\Ext(\Scon)]$ are contained.

In the above setting one can consider $\context$ and $\Scon$ as
\emph{consistent scalings} of $\mathbb{A}$. Based on this the question
for representing and quantifying \emph{inconsistencies} arises.

\begin{definition}[Conceptual Scaling Error]
  Let $\context,\Scon$ be formal contexts and $\sigma:G_{\context}\to
  G_{\Scon}$, then the \emph{conceptual scaling error} of
  $(\sigma,\Scon)$ with respect to $\context$ is the set
  $\mathcal{E}^{\context}_{\sigma,\Scon}\coloneqq
  \sigma^{-1}[\Ext(\Scon)]\setminus \Ext(\context)$.
\end{definition}

The conceptual scaling error $\mathcal{E}^{\context}_{\sigma,\Scon}$ consists of all
pre-images of closed object sets in $\Scon$ that are not
closed in the context $\context$, i.e., the object sets that
contradict the scale-measure criterion. Hence,
$\mathcal{E}^{\context}_{\sigma,\Scon}=\emptyset$ iff
$(\sigma,\Scon)\in \Sh(\context)$.

In the following, we denote by
$\sigma^{-1}[\Ext(\Scon)]|_{\Ext(\context)}\coloneqq
\sigma^{-1}[\Ext(\Scon)]\cap \Ext(\context)$ the set of consistently
reflected closed object sets of $\Scon$ by $\sigma$. This set can be
represented as the intersection of two closure systems and is thereby
a closure system as well. Using this notation together with the
canonical representation \cref{lem:appconst} we can find the following statement.

\begin{corollary}\label{cor:full}
  For $\context,\Scon$ and $\sigma: G_{\context}\mapsto G_{\Scon}$,
  there exists a scale-measure $(\psi,\Tcon)\in \Sh(\context)$ with
  $\psi^{-1}(\Ext(\Tcon))=\sigma^{-1}[\Ext(\Scon)]|_{\Ext(\context)}$.
\end{corollary}

The conceptual scaling error
$\mathcal{E}^{\context}_{\sigma,\Scon}$ does not constitute a closure system on
$G$, since it lacks the top element $G$. Moreover, the meet of
elements $A,D\in \mathcal{E}^{\context}_{\sigma,\Scon}$ can be closed in
$\context$ and thus $A\wedge D \notin \mathcal{E}^{\context}_{\sigma,\Scon}$.

To pinpoint the cause of the conceptual scaling inconsistencies we may
investigate the scale's attributes using the following proposition.

\begin{proposition}[Deciding Scale-Measures (cf. Proposition 20, \cite{navimeasure})]
  \label{prop:attr}
  Let $\context$ and $\Scon$ be two formal contexts and $\sigma:
  G_{\context} \to G_{\Scon}$, then TFAE:

  \begin{enumerate}[i)]
  \item $\sigma \text{ is a } \Scon\text{-measure of }\context$
  \item $\sigma^{-1}(m^{I_{\Scon}})\in \Ext(\context) \text{ for all } m\in M_{\Scon} $
  \end{enumerate}
\end{proposition}

Based on this result, we can decide if $(\sigma,\Scon)$ is a
scale-measure of $\context$ solely based on the attribute extents of
$\Scon$. In turn this enables us to determine the particular
attributes $n$ that cause conceptual scaling errors, i.e.
$\sigma^{-1}(n^{I_{\Scon}})\not\in \Ext(\context)$. We call the set of
all these attributes the \emph{attribute scaling error}.


\begin{corollary}\label{cor:part}
  For formal  contexts $\context, \Scon$ and map  $\sigma: G_{\context}\to
  G_{\Scon}$ let the set  $O=\{m\in M_{\Scon}\mid \sigma^{-1}(m^{I_{\Scon}})\in
  \Ext(\context)\}$. Then $(\sigma,(G_{\Scon},O,I_{\Scon}\cap
  G_{\Scon}\times O))$ is a scale-measure of $\context$.
\end{corollary}
\begin{proof}
  Follows directly from applying \cref{prop:attr}.
\end{proof}

The thus constructed scale-measure does not necessarily reflect all
extents in $\sigma^{-1}[\Ext(\Scon)]|_{\Ext(\context)}$. For this,
consider the example $\context = (\{1,2,3\},\{1,2,3\},=)$ with
$\Scon=(\{1,2,3\},\{1,2,3\},\neq)$ and the map $\id_{\{1,2,3\}}$. The
error set is equal to
$\mathcal{E}^{\context}_{\sigma,\Scon}={\{1,2,3\}\choose 2}$. Hence,
none of the scale-attributes $M_{\Scon}=\{1,2,3\}$ fulfills the
scale-measure property. By omitting the
whole set of attributes $M_{\Scon}$, we result in the context
$(G,\{\},\{\})$ whose set of extents is equal to $\{G\}$. The set
$\sigma^{-1}[\Ext(\Scon)]|_{\Ext(\context)}$ however is equal to
$\{\{\},\{1\},\{2\},\{3\},\{1,2,3\}\}$.

\begin{figure}[t]
  \centering
\hspace{-1.5cm}  \begin{minipage}{.49\linewidth}
\scalebox{1.3}{    \begin{tikzpicture}
      \draw[draw=gray!60] (1,-4.25) to[out=30, in=-30] (1,0.25);
      \draw[draw=gray!60] (-1,-4.25) to[out=150, in=-150] (-1,0.25);
      \node[draw opacity = 0,draw=white, text=gray!60] at (0,0.5)
      {$\mathcal{P}(G)$};
      \node[draw opacity = 0,draw=white, text=black] at (0,-4.5)
      {$\{G\}$};

      \node[draw opacity = 0,draw=white, text=black] at (-0.5,-1.7)
      {\scriptsize$\Ext(\context)$};
      \draw (-1,-4) to[out=150, in=-150] (-1,-2);
      \draw (0,-4) to[out=30, in=-30] (0,-2);

      \node[draw opacity = 0,draw=white, text=black] at (1,-2.2)
      {\scriptsize$\sigma^{-1}[\Ext(\Scon)]$};

      \node[draw opacity = 0,draw=white, text=black] at (0.25,-3.3)
      {\scriptsize\ $\sigma^{-1}[\Ext(\Scon)]{\wedge} \Ext(\context)$};

      \draw[dashed] (-0.1,-2.2) to[out=-30, in=90] (0.25,-3.1);
      \draw[dashed] (1,-2.4) to[out=-150, in=90] (0.25,-3.1);

      \draw[<->,draw=red] (1.1,-2.4) to (1,-3.1);
      \node[draw opacity = 0,draw=white, text=red] at (1.4,-2.9)
      {\scriptsize$\mathcal{E}^{\context}_{\sigma,\Scon}$};

    \end{tikzpicture}  }
  \end{minipage}\hspace{0.5cm}
  \begin{minipage}{.49\linewidth}
\scalebox{1.3}{ \begin{tikzpicture}
      \draw[draw=gray!60] (1,-4.25) to[out=30, in=-30] (1,0.25);
      \draw[draw=gray!60] (-1,-4.25) to[out=150, in=-150] (-1,0.25);
      \node[draw opacity = 0,draw=white, text=gray!60] at (0,0.5)
      {$\mathcal{P}(G)$};

      \node[draw opacity = 0,draw=white, text=black] at (0,-4.5)
      {$\{G\}$};

      \node[draw opacity = 0,draw=white, text=black] at (-0.5,-1.7)
      {\scriptsize$\Ext(\context)$};

      \node[draw opacity = 0,draw=white, text=black] at (1,-2.2)
      {\scriptsize\ $\sigma^{-1}[\Ext(\Scon)]$};
      \node[draw opacity = 0,draw=white, text=black] at (0.25,-3.3)
      {\tiny$\sigma^{-1}[\Ext(\Scon)]{\wedge}\Ext(\context)$};

      \draw[dashed] (-0.1,-2.2) to[out=-30, in=90] (0.25,-3.1);
      \draw[dashed] (1,-2.4) to[out=-150, in=90] (0.25,-3.1);

      \draw (0.25,-4) to[out=150, in=-150] (0.25,-2.4);
      \draw (1.25,-4) to[out=30, in=-30] (1.5,-2.4);

      \draw[draw=red,dashed] (1.1,-2.4) to[in=120,out=-30] (1.4,-2.7);
      \node[draw opacity = 0,draw=white, text=red] at (1.4,-2.9)
      {\scriptsize$\phi(\mathcal{E}^{\context}_{\sigma,\Scon})$};

    \end{tikzpicture}}
  \end{minipage}
  \caption{The conceptual scaling error and the consistent part of
    $(\sigma,\Scon)$ in $\SH(\context)$ (left). The right
    represents both parts as scale-measures of $\Scon$.}
  \label{fig:SmAsCl}
\end{figure}

\subsection{Representation and Structure of Conceptual Scaling Errors}\label{sec:3.1}
So far, we apprehended $\mathcal{E}^{\context}_{\sigma,\Scon}$ as the
set of erroneous preimages. However, the conceptual scaling error may
be represented as a part of a scale-measure:
%
%
\begin{inparaenum}[i)]
\item The first approach is to analyze the extent structure of
  $\sigma^{-1}[\Ext(\Scon)]$. This leads to a scale-measure
  $(\sigma,\Scon)$ of the apposition $\Scon\mid\context$, according to
  \cref{prop:err-measure}.  The conceptual scaling error
  $\mathcal{E}^{\context}_{\sigma,\Scon}$ is a subset of the reflected
  extents of $(\sigma,\Scon)$.  
\item The second approach is based on our result in
  \cref{cor:full}. The conceptual scaling error
  $\mathcal{E}^{\context}_{\sigma,\Scon}$ cannot be represented as
  scale-measure of $\context$. However, since
  $\mathcal{E}^{\context}_{\sigma,\Scon}\subseteq
  \sigma^{-1}[\Ext(\Scon)]$ there is a scale-measure of $\Scon$ that
  reflects $\mathcal{E}^{\context}_{\sigma,\Scon}$ (right,
  \cref{fig:SmAsCl}). Such a scale-measure can be computed using the
  canonical representation of scale-measures as highlighted by
  $\phi(\mathcal{E}^{\context}_{\sigma,\Scon})$ in \cref{fig:SmAsCl}.
  Since the scale-hierarchy is
  join-pseudocomplemented~\cite{navimeasure}, we can compute a smaller
  representation of $\sigma^{-1}[\Ext(\Scon)]|_{\Ext(\context)}$ and
  $\mathcal{E}^{\context}_{\sigma,\Scon}$. In detail, for any
  $\sigma^{-1}[\Ext(\Scon)]|_{\Ext(\context)}$ there exists a least
  element in $\SH(\Scon)$ whose join with
  $\sigma^{-1}[\Ext(\Scon)]|_{\Ext(\context)}$ yields
  $\sigma^{-1}[\Ext(\Scon)]$. Due to its smaller size, the so computed
  join-complement can be more human comprehensible than
  $\mathcal{E}^{\context}_{\sigma,\Scon}$.
\item The third option is based on splitting the scale context
  according to its consistent attributes, see~\cref{cor:part}. Both
  split elements are then considered as scale-measures of
  $\Scon$. This results in two smaller, potentially more
  comprehensible, concept lattices.
\end{inparaenum}
Additionally, all discussed scale-measures can be given in conjunctive
normalform.





\subsection{Computational Tractability}
The first thing to note, with respect to the computational
tractability, is that the size of the concept lattice of $\Scon$, as
proposed in i) (above) is larger compared to the split approaches, as
proposed in ii) and iii). This difference results in order dimensions
for the split elements that are bound by the order dimension of
$\Scon$ (cf. Proposition 24, \cite{navimeasure}). The approach in ii)
splits the scale $\Scon$ according to the conceptual scaling error
$\mathcal{E}_{\sigma,\Scon}^{\context}$, a potentially exponentially
sized problem with respect to $\Scon$. The consecutive computation of
the join-complement involves computing all meet-irreducibles in
$\sigma^{-1}[\Ext(\Scon)]$, another computationally expensive task. In
contrast, approach iii) splits $\Scon$ based on consistent attributes
and is takes therefore polynomial time in the size of
$\Scon$. However, as shown in the example after \cref{cor:part},
approach iii) may lead to less accurate representations.

\begin{figure}

\begin{minipage}{.26\linewidth}Attributes of $\context$\\

\tiny
clearing land (CL), draft (Dr), dung (Du), education (Ed), eggs (Eg),
feathers (Fe), fiber (F), fighting (Fi), guarding (G), guiding (Gu),
herding (He), horns (Ho), hunting (Hu), lawn mowing (LM), leather
(Le), manure (Ma), meat (Me), milk (Mi), mount (Mo), narcotics
detection (ND), ornamental (O), pack (Pa), pest control (PC), pets
(Pe), plowing (Pl), policing (Po), racing (Ra), rescuing (R), research
(Re), service (Se), show (Sh), skin (Sk), sport (Sp), therapy (Th),
truffle harvesting (TH), vellum (V), weed control (WC), working (W)
\end{minipage}\hspace{0.2cm}
  \begin{minipage}{0.73\linewidth}
  \hspace{-0.85cm}\scalebox{0.7}{
  \begin{cxt}
    \cxtName{$\Scon$}
    \att{0}
    \att{1}
    \att{2}
    \att{3}
    \att{4}
    \att{5}
    \att{6}
    \att{7}
    \att{8}
    \att{9}
    \obj{..x....x..}{Brahman cattle}
    \obj{..x....xxx}{European cattle}
    \obj{....x.....}{Guppy}
    \obj{.....x..xx}{alpaca}
    \obj{...x......}{bactrian camel}
    \obj{..x....x..}{bali cattle}
    \obj{.........x}{barbary dove}
    \obj{....x.....}{canary}
    \obj{....x....x}{cat}
    \obj{.x....x...}{chicken}
    \obj{x.......xx}{dog}
    \obj{...x.x.x..}{donkey}
    \obj{...x......}{dromedary}
    \obj{.x........}{duck}
    \obj{....x.....}{fancy mouse}
    \obj{....x.....}{fancy rat}
    \obj{......x...}{ferret}
    \obj{........x.}{fuegian dog}
    \obj{..x.......}{gayal}
    \obj{..x..x..xx}{goat}
  \end{cxt}}
\scalebox{0.7}{
  \begin{tabular}{|l||c|c|c|c|c|c|c|c|c|c|}
    \hline
    goldfish&\phantom{$\times$}&\phantom{$\times$}&\phantom{$\times$}&\phantom{$\times$}&\phantom{$\times$}&\phantom{$\times$}&$\times$&\phantom{$\times$}&\phantom{$\times$}&\phantom{$\times$}\\\hline
    goose&\phantom{$\times$}&$\times$&\phantom{$\times$}&\phantom{$\times$}&\phantom{$\times$}&\phantom{$\times$}&\phantom{$\times$}&\phantom{$\times$}&\phantom{$\times$}&\phantom{$\times$}\\\hline
    guinea pig&\phantom{$\times$}&\phantom{$\times$}&\phantom{$\times$}&\phantom{$\times$}&$\times$&$\times$&\phantom{$\times$}&\phantom{$\times$}&\phantom{$\times$}&$\times$\\\hline
    guineafowl&\phantom{$\times$}&$\times$&\phantom{$\times$}&\phantom{$\times$}&\phantom{$\times$}&\phantom{$\times$}&\phantom{$\times$}&\phantom{$\times$}&\phantom{$\times$}&\phantom{$\times$}\\\hline
    hedgehog&\phantom{$\times$}&\phantom{$\times$}&\phantom{$\times$}&\phantom{$\times$}&\phantom{$\times$}&\phantom{$\times$}&\phantom{$\times$}&\phantom{$\times$}&\phantom{$\times$}&\phantom{$\times$}\\\hline
    horse&\phantom{$\times$}&\phantom{$\times$}&\phantom{$\times$}&$\times$&\phantom{$\times$}&$\times$&\phantom{$\times$}&$\times$&\phantom{$\times$}&\phantom{$\times$}\\\hline
    koi&\phantom{$\times$}&\phantom{$\times$}&\phantom{$\times$}&\phantom{$\times$}&\phantom{$\times$}&\phantom{$\times$}&$\times$&\phantom{$\times$}&\phantom{$\times$}&\phantom{$\times$}\\\hline
    lama&\phantom{$\times$}&\phantom{$\times$}&\phantom{$\times$}&\phantom{$\times$}&\phantom{$\times$}&$\times$&\phantom{$\times$}&\phantom{$\times$}&$\times$&$\times$\\\hline
    mink&\phantom{$\times$}&\phantom{$\times$}&\phantom{$\times$}&\phantom{$\times$}&$\times$&\phantom{$\times$}&\phantom{$\times$}&\phantom{$\times$}&\phantom{$\times$}&\phantom{$\times$}\\\hline
    muscovy duck&\phantom{$\times$}&$\times$&\phantom{$\times$}&\phantom{$\times$}&\phantom{$\times$}&\phantom{$\times$}&\phantom{$\times$}&\phantom{$\times$}&\phantom{$\times$}&\phantom{$\times$}\\\hline
    pig&\phantom{$\times$}&\phantom{$\times$}&\phantom{$\times$}&\phantom{$\times$}&$\times$&$\times$&$\times$&\phantom{$\times$}&$\times$&\phantom{$\times$}\\\hline
    pigeon&\phantom{$\times$}&\phantom{$\times$}&\phantom{$\times$}&\phantom{$\times$}&\phantom{$\times$}&\phantom{$\times$}&$\times$&\phantom{$\times$}&\phantom{$\times$}&\phantom{$\times$}\\\hline
    rabbit&\phantom{$\times$}&\phantom{$\times$}&\phantom{$\times$}&\phantom{$\times$}&$\times$&$\times$&\phantom{$\times$}&\phantom{$\times$}&\phantom{$\times$}&$\times$\\\hline
    sheep&\phantom{$\times$}&\phantom{$\times$}&$\times$&\phantom{$\times$}&$\times$&$\times$&$\times$&\phantom{$\times$}&$\times$&\phantom{$\times$}\\\hline
    silkmoth&\phantom{$\times$}&\phantom{$\times$}&\phantom{$\times$}&\phantom{$\times$}&$\times$&\phantom{$\times$}&\phantom{$\times$}&\phantom{$\times$}&\phantom{$\times$}&\phantom{$\times$}\\\hline
    silver fox&\phantom{$\times$}&\phantom{$\times$}&\phantom{$\times$}&\phantom{$\times$}&$\times$&\phantom{$\times$}&\phantom{$\times$}&\phantom{$\times$}&\phantom{$\times$}&\phantom{$\times$}\\\hline
    society finch&\phantom{$\times$}&\phantom{$\times$}&\phantom{$\times$}&\phantom{$\times$}&$\times$&\phantom{$\times$}&\phantom{$\times$}&\phantom{$\times$}&\phantom{$\times$}&\phantom{$\times$}\\\hline
    striped skunk&\phantom{$\times$}&\phantom{$\times$}&\phantom{$\times$}&\phantom{$\times$}&$\times$&\phantom{$\times$}&\phantom{$\times$}&\phantom{$\times$}&\phantom{$\times$}&\phantom{$\times$}\\\hline
    turkey&\phantom{$\times$}&$\times$&\phantom{$\times$}&\phantom{$\times$}&\phantom{$\times$}&\phantom{$\times$}&\phantom{$\times$}&\phantom{$\times$}&\phantom{$\times$}&\phantom{$\times$}\\\hline
    water buffalo&\phantom{$\times$}&\phantom{$\times$}&$\times$&\phantom{$\times$}&\phantom{$\times$}&\phantom{$\times$}&\phantom{$\times$}&$\times$&\phantom{$\times$}&$\times$\\\hline
    yak&\phantom{$\times$}&\phantom{$\times$}&\phantom{$\times$}&$\times$&\phantom{$\times$}&\phantom{$\times$}&\phantom{$\times$}&$\times$&$\times$&\phantom{$\times$}\\\bottomrule
  \end{tabular}}
  \end{minipage}


  \caption{Attributes of the \emph{Domestic} data set (left) and a
    factor (right).}
  \label{fig:factor}
\end{figure}

\section{Experimental Evaluation} 
To provide practical evidence for the applicability of the just
introduced conceptual scaling error, we conducted an experiment on
eleven data sets. In those we compared the classical errors, such as
Frobenius norm, to the conceptual scaling error. The data sets were, if not
otherwise specified, nominally scaled to a binary representation. Six
of them are available through the UCI\footnote{\scriptsize i)~\url{https://archive.ics.uci.edu/ml/datasets/Acute+Inflammations},\\
  ii)~\url{https://archive.ics.uci.edu/ml/datasets/Hayes-Roth},\\
  iii)~\url{https://archive.ics.uci.edu/ml/datasets/zoo},\\ 
  iv)~\url{https://archive.ics.uci.edu/ml/datasets/mushroom},\\
  v)~\url{https://archive.ics.uci.edu/ml/datasets/HIV-1+protease+cleavage},\\
  vi)~\url{https://archive.ics.uci.edu/ml/datasets/Plants} and
  \url{https://plants.sc.egov.usda.gov/java/}} \cite{uci} data sets
repository:
\begin{inparaenum}[i)]
\item \emph{Diagnosis} \cite{diagnosis} with \emph{temperature} scaled in intervals of [35.0 37.5) [37.5 40.0) [40.0 42.0],
\item \emph{Hayes-Roth} 
\item \emph{Zoo} 
\item \emph{Mushroom}
\item \emph{HIV-1ProteaseCleavage} \cite{HIV} and
\item \emph{Plant-Habitats}
\end{inparaenum}
four kaggle\footnote{\scriptsize \url{https://www.kaggle.com/},
  i)~\url{https://www.kaggle.com/odartey/top-chess-players} and
  \url{https://www.fide.com/},
ii)~\url{https://www.kaggle.com/brittabettendorf/berlin-airbnb-data/},
iii)~\url{https://www.kaggle.com/rajeevw/ufcdata},\\
iv)~\url{https://www.kaggle.com/shuyangli94/food-com-recipes-and-user-interactions}} data sets
\begin{inparaenum}[i)]
\item[vii)] \emph{Top-Chess-Players} with \emph{rating,rank,games,bith\_year} ordinally scaled,
\item[viii)] neighbourhood data from the \emph{Airbnb-Berlin} data sets,
\item[ix)] \emph{A\_fighter} and \emph{B\_fighter} from the
  \emph{UFC-Fights} data sets and
\item[x)] \emph{Recipes} \cite{recipes}.
\end{inparaenum}
The eleventh data set is generated from the Wikipedia list of
\emph{Domesticated Animals}.\footnote{\scriptsize
\url{https://en.wikipedia.org/w/index.php?title=List_of_domesticated_animals},
25.02.2020} This data set is also used for a qualitative analysis. We
summarized all data sets in~\cref{tab:error}, first and second major column.

\begin{table}[t]
  \centering
  \caption{Quantifying the Conceptual Scaling Error for approximations
    $\context_{\approx} = \Scon{\cdot}H$ of data sets $\context$ by
    Binary Matrix Factorization. Cells with '-' where not computed due
    to computational intractability. Density (D), Attribute Error
    (AE), Conceptual Scaling Error (CE), Hemming distance between
    $I_{\context}$ and $I_{\Scon{\cdot}H}$ relative to $|G|{\cdot}|M|$
     (H\%), Frobenius measure between
    $I_{\context}$ and $I_{\Scon{\cdot}H}$. }
\scalebox{0.8}{  \begin{tabular}{|l||r|r|r|r||r|r|r|r|r||r|r|r|r|r|}
    \hline
    &
    \multicolumn{4}{|c||}{Context $\context$} & 
    \multicolumn{5}{|c||}{Approximated Context $\context_{\approx} = \Scon{\cdot}H$} &
    \multicolumn{5}{|c|}{Respective Scale $\Scon$}\\

    &\multicolumn{1}{c|}{$|G|$} &\multicolumn{1}{|c|}{ $|M|$} &\multicolumn{1}{|c|}{ D} &\multicolumn{1}{|c||}{ $|\BV|$} &\multicolumn{1}{|c|}{Frob} &\multicolumn{1}{|c|}{ H\% }& \multicolumn{1}{|c|}{$|\BV|$}
    &\multicolumn{1}{|c|}{AE} &\multicolumn{1}{|c||}{CE} &\multicolumn{1}{|c|}{$|M|$} & \multicolumn{1}{|c|}{D} & \multicolumn{1}{|c|}{$|\BV|$}& \multicolumn{1}{|c|}{AE} & \multicolumn{1}{|c|}{CE} \\   \hline
Diagnosis&120&17&0.471&88&13.04&8.3&26&6&7&4&0.250&6&0&0\\
Hayes-Roth&132&18&0.218&215&16.40&11.3&33&8&26&4&0.350&12&3&8\\
Domestic&41&55&0.158&292&8.49&3.2&148&14&68&10&0.183&34&6&15\\
Zoo&101&43&0.395&4579&15.52&5.5&442&13&347&7&0.315&25&2&4\\
Chess&346&683&0.473&3211381&82.01&2.8&229585&246&224673&26&0.767&4334&24&4280\\
Mushroom&8124&119&0.193&238710&243.86&6.2&10742&48&10598&11&0.277&139&7&116\\
HIV-1PC&6590&162&0.055&115615&221.38&4.6&303&32&229&13&0.154&330&12&236\\
Plant-Habitats&34781&68&0.127&-&322.16&4.4&-&68&-&8&0.128&256&8&255\\
Airbnb-Berlin&22552&145&0.007&-&130.29&0.5&-&0&0&12&0.057&8&1&1\\
UFC-Fights&5144&1915&0.001&-&101.43&0.1&2&0&1&44&0.932&2&44&1\\
Recipes&178265&58&0.057&-&492.8&2.3&-&7&-&8&0.236&256&4&187 \\
    \hline
  \end{tabular}}
  \label{tab:error}
\end{table}

As dimension reduction method, we employ the \emph{binary matrix
factorization}~\cite{bmf} of the Nimfa~\cite{nimfa} framework. Their
algorithm is an adaption of the \emph{non-negative matrix
factorization}(NMF). In addition to the regular NMF a penalty and a
thresholding function are applied to binarize the output.  For any
given data set $\context$ two binary factors $\Scon,H$ with $\context
\approx \Scon \cdot H$ are computed.

The BMF factorization algorithm takes several parameters, such as
convergence $\lambda_{w},\lambda_{h}$, which we left at their default
value of 1.1. We increased the maximum number of iterations to 500 to
ensure convergence and conducted ten runs, of which took the best
fit. The target number of attribute (features) in $|M_{\Scon}|$ was
set approximately to $\sqrt{|M_{\context}|}$ to receive a data
dimension reduction of one magnitude.

We depicted the results, in particular the quality of the
factorizations, in~\cref{tab:error} (major column three and four)
where $\context_{\approx}=\Scon \cdot H$ is the BMF approximation of
$\context$. Our investigation considers standard measures, such as
Frobenius norm (\emph{Frob}) and relative Hamming distance
(\emph{H\%}), as well as the proposed conceptual scaling error
(\emph{CE}) and attribute scaling error (\emph{AE}). For the large
data sets, i.e., the last four in~\cref{tab:error}, we omitted
computing the number of concepts due to its computational
intractability, indicated by '-'. Therefore we were not able to
compute the conceptual scaling errors of the approximate data sets
$\context_\approx$. However, the conceptual scaling error of the
related scales $\Scon$ is independent of the computational
tractability of CE of $\context_\approx$.

We observe that the values for Frob and for H\% differ vastly among
the different data sets. For Example H\% varies from 0.1 to 11.3.  We
find that for all data sets $|\mathfrak{B}(\context)|$ is substantially larger
than $|\mathfrak{B}(\context_{\approx})|$, independently of the values of Frob
and H\%. Hence, BMF leads to a considerable loss of concepts.  When
comparing the conceptual and attribute scaling error to Frob and H\%,
we observe that the novel conceptual errors capture different aspects
than the classical matrix norm error.  For example, \emph{Domestic}
and \emph{Chess} have similar values for H\%, however, their error
values with respect to attributes and concepts differ
significantly. In detail, the ratio of $|CE|/|\mathfrak{B}(\context_\approx)$ is
0.98 for \emph{Chess} and 0.46 for \emph{Domestic}, and the ratio for
$|AE|/|M_\approx|$ is 0.36 for \emph{Chess} and .25 for
\emph{Domestic}.

While we do not know the number of concepts for \emph{Airbnb-Berlin},
we do know that conceptual scaling error of the related
$\context_{\approx}$ is 0 due to AE being 0 and \cref{prop:attr}. The
factorization of the \emph{UFC-Fights} produced an empty context
$\context_{\approx}$. Therefore, all attribute derivations in
$\context_{\approx}$ are the empty set, whose pre-image is an extent
of $\context$, hence, AE is 0. We suspect that BMF is unable to cope
with data sets that exhibit a very low density. It is noteworthy that
we cannot elude this conclusion from the value of the Frob and H\%. By
investigating the binary factor $\Scon$ using the conceptual scaling
error and the attribute error, we are able to detect the occurrence of
this phenomenon. In detail, we see that 44 out of 44 attributes are
inconsistent.

We can take from our investigation that low H\% and Frob values do not
guaranty good factorizations with respect to preserving the conceptual
structure of the original data set. In contrast, we claim that the
proposed scaling errors are capable of capturing such error to some
extent.  On a final note, we may point out that the conceptual scaling
errors enable a quantifiable comparison of a scaling $\Scon$ to the
original data set $\context$, despite different dimensionality.

\subsection{Qualitative Analysis}
The \emph{domestic} data set includes forty-one animals as objects and
fifty-five purposes for their domestication as attributes, such as
\emph{pets, hunting, meat, etc}. The resulting $\context$ has a total
of 2255 incidences and the corresponding concept lattice has 292
formal concepts. We applied the BMF algorithm as before, which
terminated after 69 iterations with the scale depicted
in~\cref{fig:factor}. The incidence of $\context_\approx\coloneqq\Scon
\cdot H$ has seventy-three wrong incidences, i.e., wrongfully present
or absent pairs, which results in H\% of 3.2. The corresponding concept lattice of $\context_\approx$ has 148 concepts, which is nearly half of $\mathfrak{B}(\context)$. Furthermore, out
of these 148 concepts there are only 80 correct, i.e.,  in
$\Ext(\context)$. This results in a conceptual scaling error of 68, which is in particular interesting in the light of the apparently low H\% error.

To pinpoint the particular errors, we employ i)-iii) from
\cref{sec:3.1}. The result of the first approach is visualized in
\cref{fig:err1} and displays the concept lattice of
$\sigma^{-1}[\Ext(\Scon)]$ in which the elements of
$\mathcal{E}^{\context}_{\sigma,\Scon}$ are highlighted in red. First,
we notice in the lattice diagram that the inconsistent extents
$\mathcal{E}^{\context}_{\sigma,\Scon}$ are primarily in the upper
part. Seven out of fifteen are derivations from attribute combinations
of $9,8$, and $5$. This indicates that the factorization was
especially inaccurate for those attributes. The attribute extents of
$6,4$, and $2$ are in $\mathcal{E}^{\context}_{\sigma,\Scon}$,
however, many of their combinations with other attributes result in
extents of $\Ext(\context)$.

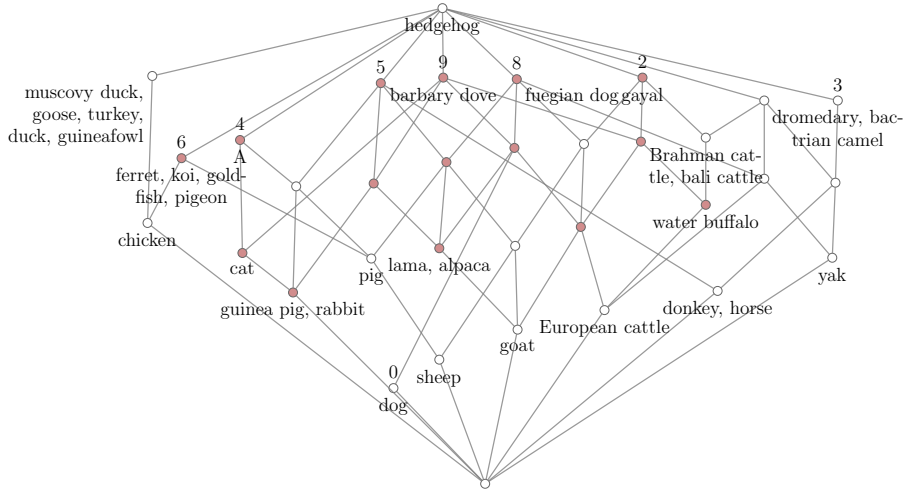
\begin{figure}[t]
  \hspace{-1cm}\scalebox{0.6}{\colorlet{mivertexcolor}{white}
\colorlet{jivertexcolor}{red!60!black}
\colorlet{vertexcolor}{white}
\colorlet{bordercolor}{black!60}
\colorlet{linecolor}{gray}
\tikzset{vertexbase/.style={semithick, shape=circle, inner sep=2pt, outer sep=0pt, draw=bordercolor},%
  vertex/.style={vertexbase, fill=vertexcolor!45},%
  mivertex/.style={vertexbase, fill=mivertexcolor!45},%
  jivertex/.style={vertexbase, fill=jivertexcolor!45},%
  divertex/.style={vertexbase, top color=mivertexcolor!45, bottom color=jivertexcolor!45},%
  conn/.style={-, thick, color=linecolor}%
}
\tikzstyle{l} = [text width=4cm,align=right]
\tikzstyle{n} = [text width=4cm,align=center]
\tikzstyle{r} = [text width=4cm,align=left]
\begin{tikzpicture}[scale=0.25, label distance = 0.1cm, font=\large]
  \begin{scope} 
    \begin{scope} 
      \foreach \nodename/\nodetype/\xpos/\ypos in {%
        0/vertex/-0.4043436293438205/11.618725868725917,
        1/vertex/-8.533301158301342/20.105791505791558,
        2/vertex/-4.4688223938225775/22.628571428571483,
        3/vertex/2.4688223938222436/25.29150579150584,
        4/vertex/10.171718146717964/27.024517374517423,
        5/jivertex/-17.4331081081083/28.585135135135175,
        6/vertex/20.198359073358887/28.725289575289622,
        7/vertex/-10.495463320463507/31.598455598455644,
        8/vertex/30.35955598455579/31.66853281853286,
        9/jivertex/-21.918050193050384/32.08899613899619,
        10/jivertex/-4.4688223938225065/32.5094594594595,
        11/vertex/2.25859073359058/32.71969111969116,
        12/jivertex/8.074999999999811/34.40154440154444,
        13/vertex/-30.32731660231679/34.751930501930545,
        14/jivertex/19.14720077220059/36.363706563706614,
        15/vertex/-17.152799227799413/38.0,
        16/jivertex/-10.285231660231844/38.25579150579155,
        17/vertex/30.63986486486469/38.32586872586877,
        18/vertex/24.332915057914875/38.676254826254876,
        19/jivertex/-3.8381274131276015/40.147876447876484,
        20/jivertex/-27.31399613899633/40.498262548262595,
        21/jivertex/2.1885135135133282/41.40926640926646,
        22/vertex/8.35530888030869/41.759652509652554,
        23/jivertex/13.400868725868548/41.96988416988422,
        24/jivertex/-22.12828185328204/42.11003861003866,
        25/vertex/19.147200772200584/42.32027027027031,
        26/vertex/-29.906853281853472/47.791891891891936,
        27/vertex/24.332915057914867/45.613899613899655,
        28/vertex/30.85009652509634/45.613899613899655,
        29/jivertex/-9.654536679536868/47.1555984555985,
        30/jivertex/2.3987451737449845/47.5059845559846,
        31/jivertex/-4.118436293436481/47.64613899613904,
        32/jivertex/13.541023166022974/47.64613899613904,
        33/vertex/-4.188513513513705/53.80733590733595
      } \node[\nodetype] (\nodename) at (\xpos, \ypos) {};
    \end{scope}        
    \begin{scope} 
      \path (0) edge[conn,draw opacity=0.8] (13);
      \path (6) edge[conn,draw opacity=0.8] (29);
      \path (25) edge[conn,draw opacity=0.8] (27);
      \path (0) edge[conn,draw opacity=0.8] (4);
      \path (12) edge[conn,draw opacity=0.8] (22);
      \path (19) edge[conn,draw opacity=0.8] (30);
      \path (10) edge[conn,draw opacity=0.8] (16);
      \path (14) edge[conn,draw opacity=0.8] (23);
      \path (0) edge[conn,draw opacity=0.8] (6);
      \path (3) edge[conn,draw opacity=0.8] (10);
      \path (11) edge[conn,draw opacity=0.8] (19);
      \path (18) edge[conn,draw opacity=0.8] (27);
      \path (0) edge[conn,draw opacity=0.8] (1);
      \path (3) edge[conn,draw opacity=0.8] (11);
      \path (0) edge[conn,draw opacity=0.8] (3);
      \path (13) edge[conn,draw opacity=0.8] (20);
      \path (29) edge[conn,draw opacity=0.8] (33);
      \path (26) edge[conn,draw opacity=0.8] (33);
      \path (24) edge[conn,draw opacity=0.8] (33);
      \path (17) edge[conn,draw opacity=0.8] (28);
      \path (4) edge[conn,draw opacity=0.8] (18);
      \path (5) edge[conn,draw opacity=0.8] (16);
      \path (16) edge[conn,draw opacity=0.8] (31);
      \path (27) edge[conn,draw opacity=0.8] (33);
      \path (13) edge[conn,draw opacity=0.8] (26);
      \path (6) edge[conn,draw opacity=0.8] (17);
      \path (9) edge[conn,draw opacity=0.8] (24);
      \path (23) edge[conn,draw opacity=0.8] (32);
      \path (15) edge[conn,draw opacity=0.8] (29);
      \path (11) edge[conn,draw opacity=0.8] (22);
      \path (1) edge[conn,draw opacity=0.8] (21);
      \path (28) edge[conn,draw opacity=0.8] (33);
      \path (10) edge[conn,draw opacity=0.8] (21);
      \path (4) edge[conn,draw opacity=0.8] (12);
      \path (5) edge[conn,draw opacity=0.8] (15);
      \path (21) edge[conn,draw opacity=0.8] (31);
      \path (19) edge[conn,draw opacity=0.8] (29);
      \path (3) edge[conn,draw opacity=0.8] (12);
      \path (22) edge[conn,draw opacity=0.8] (30);
      \path (16) edge[conn,draw opacity=0.8] (29);
      \path (30) edge[conn,draw opacity=0.8] (33);
      \path (0) edge[conn,draw opacity=0.8] (2);
      \path (17) edge[conn,draw opacity=0.8] (27);
      \path (22) edge[conn,draw opacity=0.8] (32);
      \path (12) edge[conn,draw opacity=0.8] (23);
      \path (2) edge[conn,draw opacity=0.8] (7);
      \path (4) edge[conn,draw opacity=0.8] (14);
      \path (0) edge[conn,draw opacity=0.8] (5);
      \path (15) edge[conn,draw opacity=0.8] (24);
      \path (18) edge[conn,draw opacity=0.8] (30);
      \path (25) edge[conn,draw opacity=0.8] (32);
      \path (23) edge[conn,draw opacity=0.8] (31);
      \path (31) edge[conn,draw opacity=0.8] (33);
      \path (7) edge[conn,draw opacity=0.8] (20);
      \path (10) edge[conn,draw opacity=0.8] (19);
      \path (12) edge[conn,draw opacity=0.8] (21);
      \path (32) edge[conn,draw opacity=0.8] (33);
      \path (8) edge[conn,draw opacity=0.8] (18);
      \path (9) edge[conn,draw opacity=0.8] (31);
      \path (7) edge[conn,draw opacity=0.8] (19);
      \path (20) edge[conn,draw opacity=0.8] (33);
      \path (7) edge[conn,draw opacity=0.8] (15);
      \path (14) edge[conn,draw opacity=0.8] (25);
      \path (5) edge[conn,draw opacity=0.8] (9);
      \path (0) edge[conn,draw opacity=0.8] (8);
      \path (8) edge[conn,draw opacity=0.8] (17);
      \path (21) edge[conn,draw opacity=0.8] (30);
      \path (2) edge[conn,draw opacity=0.8] (11);
    \end{scope}
    \begin{scope} 
      \foreach \nodename/\labelpos/\labelopts/\labelcontent in {%
        1/below/n/{dog},
        1/above/n/{0}, 
        2/below/n/{sheep},
        3/below/n/{goat},
        4/below/n/{European cattle},
        5/below/n/{guinea pig, rabbit},
        6/below/n/{donkey, horse},
        7/below/n/{pig},
        8/below/n/{yak},
        9/below/n/{cat},
        10/below/n/{lama, alpaca},
        13/below/n/{chicken},
        14/below/n/{water buffalo},
        20/below/n/{ferret, koi, goldfish, pigeon},
        20/above/n/{6},
        24/below/n/{A},
        24/above/n/{4},
        25/below/n/{Brahman cattle, bali cattle},
        26/below left/l/{muscovy duck, goose, turkey, duck, guineafowl},
        28/below/n/{dromedary, bactrian camel},
        28/above/n/{3},
        29/above/n/{5},
        30/below right/r/{fuegian dog},
        30/above/n/{8},  
        31/below/n/{barbary dove},
        31/above/n/{9},
        32/below/n/{gayal},
        32/above/n/{2},
        33/below/n/{hedgehog}
      } \coordinate[label={[\labelopts]\labelpos:{\labelcontent}}](c) at (\nodename);
    \end{scope}
  \end{scope}
\end{tikzpicture}}
  \caption{Concept lattice of the \emph{Domestic} scale context. $\mathcal{E}^{\context}_{\sigma,\Scon}$ is
    indicated in red.
  {A=$\{$society finch, silkmoth, fancy mouse, mink, fancy rat, striped skunk, Guppy, canary, silver fox$\}$}}
  \label{fig:err1}
\end{figure}

The resulting lattices of applying approach ii) are depicted in
\cref{fig:err2}, the consistent lattice of 
$\sigma^{-1}[\Ext(\Scon)]|_{\Ext(\context)}$ on the left and its join-complement on the right.
The consistent part has nineteen concepts, all depicted attributes can
considered in conjunctive normalform. The join-complement consists of
twenty-two concepts of which the incorrect ones are marked in red.

Based on this representation, we can see that twenty out of the
forty-one objects have no associated attributes.  These include
objects like \emph{lama, alpaca} or \emph{barbary dove}, which we have
also indirectly identified by i) as derivations of
$5,8,9$. Furthermore, we see that thirteen out of the fifty-five
attributes of $\context$ are not present in any conjunctive
attributes.  These attributes include domestication purposes like
\emph{tusk, fur}, or \emph{hair}. Out of our expertise we suppose that
these could form a meaningful cluster in the specific data realm.  In
the join-complement, we can identify the attributes 5,8,9 as being
highly inconsistently scaled, as already observed in the paragraph
above.

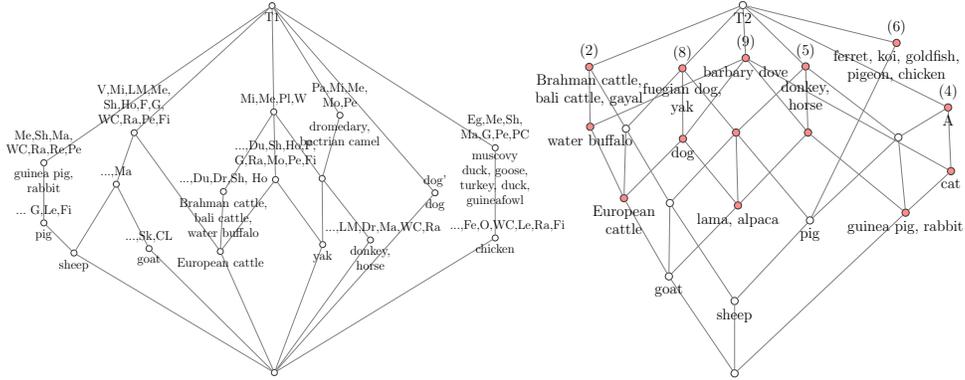
\begin{figure}[t]
  \begin{minipage}{0.5\linewidth}
    \hspace{-1cm}\scalebox{0.4}{\colorlet{mivertexcolor}{white}
\colorlet{jivertexcolor}{white}
\colorlet{vertexcolor}{white}
\colorlet{bordercolor}{black!80}
\colorlet{linecolor}{gray}
\tikzset{vertexbase/.style={semithick, shape=circle, inner sep=2pt, outer sep=0pt, draw=bordercolor},%
  vertex/.style={vertexbase, fill=vertexcolor!45},%
  mivertex/.style={vertexbase, fill=mivertexcolor!45},%
  jivertex/.style={vertexbase, fill=jivertexcolor!45},%
  divertex/.style={vertexbase, top color=mivertexcolor!45, bottom color=jivertexcolor!45},%
  conn/.style={-, thick, color=linecolor}%
}
\tikzstyle{n2} = [text width=4cm,align=center,label distance=0.3cm]
\tikzstyle{n} = [text width=3cm,align=center]
\begin{tikzpicture}[scale=0.5,label distance = 0.1cm,font=\large]
  \begin{scope} 
    \begin{scope} 
      \foreach \nodename/\nodetype/\xpos/\ypos in {%
        0/vertex/1.3218146718146713/6.042084942084941,
        1/vertex/-12.0/14.0,
        2/vertex/-2.271235521235525/14.101158301158335,
        3/vertex/-7.011003861003864/14.292277992278045,
        4/vertex/4.532625482625484/14.559845559845584,
        5/vertex/7.705212355212353/14.86563706563711,
        6/vertex/16.0/15.0,
        7/vertex/-14.0/16.0,
        8/vertex/12.0/18.0,
        9/vertex/-9.189768339768346/18.573359073359125,
        10/vertex/-2.0654440154440174/18.088030888030926,
        11/vertex/1.3982625482625437/18.879150579150597,
        12/vertex/4.494401544401541/18.955598455598476,
        13/vertex/16.0/21.0,
        14/vertex/-14.0/20.0,
        15/vertex/-8.0/22.0,
        16/vertex/5.679343629343627/23.16023166023172,
        17/vertex/1.283590733590735/23.389575289575305,
        18/vertex/1.1689189189189193/30.44864864864865
      } \node[\nodetype] (\nodename) at (\xpos, \ypos) {};
    \end{scope}
    \begin{scope} 
      \path (16) edge[conn] (18);
      \path (13) edge[conn] (18); 
      \path (0) edge[conn] (8);
      \path (9) edge[conn] (15);
      \path (6) edge[conn] (13);
      \path (14) edge[conn] (18);
      \path (0) edge[conn] (3);
      \path (12) edge[conn] (16);
      \path (8) edge[conn] (18);
      \path (4) edge[conn] (11);
      \path (15) edge[conn] (18);
      \path (17) edge[conn] (18);
      \path (7) edge[conn] (14);
      \path (5) edge[conn] (12);
      \path (0) edge[conn] (6);
      \path (0) edge[conn] (5);
      \path (10) edge[conn] (17);    
      \path (4) edge[conn] (12);
      \path (2) edge[conn] (10);
      \path (1) edge[conn] (7);
      \path (1) edge[conn] (9);
      \path (2) edge[conn] (15);
      \path (0) edge[conn] (1);
      \path (2) edge[conn] (11);
      \path (0) edge[conn] (4);
      \path (3) edge[conn] (9);
      \path (12) edge[conn] (17);
      \path (0) edge[conn] (2);
      \path (11) edge[conn] (17);
    \end{scope}                                             
    \begin{scope} 
      \foreach \nodename/\labelpos/\labelopts/\labelcontent in {%
        1/below/n/{sheep},
        2/below/n/{European cattle},
        3/below/n/{goat},
        3/above/n/{...,Sk,CL},
        4/below/n/{yak},
        5/below/n/{donkey,\\ horse},
        5/above/n/{...,LM,Dr,Ma,WC,Ra},
        6/below/n/{chicken},
        6/above/n/{...,Fe,O,WC,Le,Ra,Fi},
        7/below/n/{pig},
        7/above/n/{... G,Le,Fi},
        8/below/n/{dog},
        8/above/n/{dog'},
        9/above/n/{...,Ma},
        10/below/n/{Brahman cattle, bali cattle, water buffalo},
        10/above/n/{...,Du,Dr,Sh, Ho},
        11/above/n2/{...,Du,Sh,Ho,F, G,Ra,Mo,Pe,Fi},
        13/below/n/{muscovy duck, goose, turkey, duck, guineafowl},
        13/above/n/{Eg,Me,Sh,\\ Ma,G,Pe,PC},
        14/below/n/{guinea pig, rabbit},
        14/above/n/{Me,Sh,Ma, WC,Ra,Re,Pe},
        15/above/n/{V,Mi,LM,Me, Sh,Ho,F,G, WC,Ra,Pe,Fi},
        16/below/n/{dromedary, bactrian camel},
        16/above/n/{Pa,Mi,Me, Mo,Pe},
        17/above/n/{Mi,Me,Pl,W},
        18/below/n/{T1}
      } \coordinate[label={[\labelopts]\labelpos:{\labelcontent}}](c) at (\nodename);
    \end{scope}
  \end{scope}
\end{tikzpicture}}
  \end{minipage}\hspace{-5mm}
  \begin{minipage}{.49\linewidth}
    \scalebox{0.48}{\colorlet{mivertexcolor}{white}
\colorlet{jivertexcolor}{red}
\colorlet{vertexcolor}{mivertexcolor!50}
\colorlet{bordercolor}{black!80}
\colorlet{linecolor}{gray}
\tikzset{vertexbase/.style={semithick, shape=circle, inner sep=2pt, outer sep=0pt, draw=bordercolor},%
  vertex/.style={vertexbase, fill=vertexcolor!45},%
  mivertex/.style={vertexbase, fill=mivertexcolor!45},%
  jivertex/.style={vertexbase, fill=jivertexcolor!45},%
  divertex/.style={vertexbase, top color=mivertexcolor!45, bottom color=jivertexcolor!45},%
  conn/.style={-, thick, color=linecolor}%
}
\tikzstyle{n2} = [text width=4cm,align=center,label distance=0.3cm]
\tikzstyle{l} = [text width=4cm,align=right]
\tikzstyle{n} = [text width=4cm,align=center]
\tikzstyle{r} = [text width=4cm,align=left]
\begin{tikzpicture}[scale=0.3, label distance=0.1cm, font=\large]
  \begin{scope} 
    \begin{scope} 
      \foreach \nodename/\nodetype/\xpos/\ypos in {%
        0/vertex/0.6912162162162261/4.94594594594605,
        1/vertex/0.6912162162162261/11.635135135135243,
        2/vertex/-5.37364864864864/13.909459459459569,
        3/vertex/7.647972972972987/19.082432432432544,
        4/jivertex/16.47770270270272/19.79594594594606,
        5/jivertex/1.003378378378386/20.46486486486498,
        6/vertex/-5.284459459459455/20.68783783783795,
        7/jivertex/-9.52094594594594/21.133783783783898,
        8/jivertex/20.669594594594614/23.631081081081195,
        9/jivertex/-4.080405405405401/26.618918918919036,
        10/vertex/15.808783783783802/26.75270270270282,
        11/jivertex/0.82500000000001/27.198648648648764,
        12/jivertex/7.469594594594611/27.198648648648764,
        13/vertex/-9.342567567567563/27.555405405405523,
        14/jivertex/-12.642567567567564/27.7337837837839,
        15/jivertex/20.402027027027046/29.517567567567685,
        16/jivertex/15.645270270270288/35.479729729729847,
        17/jivertex/1.7168918918919047/34.08513513513525,
        18/jivertex/-4.124999999999993/33.129729729729846,
        19/jivertex/-12.731756756756752/33.26351351351363,
        20/jivertex/7.246621621621635/33.30810810810823,
        21/vertex/1.449324324324337/38.97162162162174
      } \node[\nodetype] (\nodename) at (\xpos, \ypos) {};
    \end{scope}
    \begin{scope} 
      \path (20) edge[conn] (21);
      \path (5) edge[conn] (11);
      \path (8) edge[conn] (15);
      \path (16) edge[conn] (21);
      \path (14) edge[conn] (17);
      \path (6) edge[conn] (13);
      \path (1) edge[conn] (6);
      \path (11) edge[conn] (20);
      \path (7) edge[conn] (14);
      \path (13) edge[conn] (19);
      \path (1) edge[conn] (3);
      \path (13) edge[conn] (18);
      \path (11) edge[conn] (18);
      \path (9) edge[conn] (18);
      \path (6) edge[conn] (11);
      \path (14) edge[conn] (19);
      \path (10) edge[conn] (15);
      \path (2) edge[conn] (7);
      \path (19) edge[conn] (21);
      \path (4) edge[conn] (10);
      \path (0) edge[conn] (2);
      \path (3) edge[conn] (11);
      \path (7) edge[conn] (13);
      \path (10) edge[conn] (20);
      \path (2) edge[conn] (5);
      \path (4) edge[conn] (12);
      \path (9) edge[conn] (17);
      \path (2) edge[conn] (6);
      \path (17) edge[conn] (21);
      \path (3) edge[conn] (16);
      \path (18) edge[conn] (21);
      \path (0) edge[conn] (1);
      \path (5) edge[conn] (12);
      \path (12) edge[conn] (17);
      \path (8) edge[conn] (17);
      \path (0) edge[conn] (4);
      \path (3) edge[conn] (10);
      \path (12) edge[conn] (20);
      \path (5) edge[conn] (9);
      \path (4) edge[conn] (8);
      \path (15) edge[conn] (21);
      \path (7) edge[conn] (9);
    \end{scope}
    \begin{scope} 
      \foreach \nodename/\labelpos/\labelopts/\labelcontent in {%
        1/below/n/{sheep},
        2/below/n/{goat},
        3/below/n/{pig},
        4/below/n/{guinea pig, rabbit},
        5/below/n/{lama, alpaca},
        7/below/n/{European\\ cattle},
        8/below/n/{cat},
        9/below/n/{dog},
        14/below/n/{water buffalo},
        15/below/n/{A},
        15/above/n/{(4)},
        16/below/n/{ferret, koi, goldfish, pigeon, chicken},
        16/above/n/{(6)},
        17/below/n/{barbary dove},
        17/above/n/{(9)},
        18/below/n2/{fuegian dog,\\ yak},
        18/above/n/{(8)},
        19/below/n/{Brahman cattle,\\ bali cattle, gayal},
        19/above/n/{(2)},
        20/below/n2/{donkey,\\ horse},
        20/above/n/{(5)},
        21/below/n/{T2}
      } \coordinate[label={[\labelopts]\labelpos:{\labelcontent}}](c) at (\nodename);
    \end{scope}
  \end{scope}
\end{tikzpicture}}
  \end{minipage}
  \caption{The concept lattice of all valid extents of the scale
    $\Scon$ (left, in conjunctive normalform in $\SH(\context)$) and
    its join-complement (right) of the
    \emph{Domestic} data set. Extents in the lattice drawing of the
    join-complement that are extents not in the \emph{Domestic} context are highlighted in red.
    {dog'=$\{$Ed,TH,Pa,Gu,He,Dr,Sp,Me,Sh,Po,F,W,Re,G,ND,Le,Ra,Hu,Se,Th,Pe,PC,Fi$\}$ A=$\{$society finch, silkmoth, fancy mouse,
      mink, fancy rat, striped skunk, Guppy, canary, silver fox$\}$,
      T1=$\{$fuegian dog, lama, ferret, alpaca, society finch,
      silkmoth, fancy mouse, koi, hedgehog, mink, fancy rat, striped
      skunk, goldfish, barbary dove, Guppy, canary, pigeon, silver
      fox, cat, gayal$\}$, T2=$\{$dromedary, muscovy duck, bactrian
      camel, goose, turkey, hedgehog, duck, guineafowl$\}$}}
  \label{fig:err2}
\end{figure}

The third approach results in a scale $\Scon_{O}$ of four consistent
attributes and a scale $\Scon_{\hat{O}}$ of six non-consistent
attributes. The scale $\Scon_{O}$ has seven concepts and the scale
$\Scon_{\hat{O}}$ has twenty-two. We may note that the concept lattice
of $\Scon_{\hat{O}}$ is identical to the join-complement of the
previous approach. In general this is not the case, one cannot even
assume isomorphy. While the concept lattice of $\Scon_{O}$ does miss
some of the consistent extents of
$\sigma^{-1}[\Ext(\Scon)]|_{\Ext(\context)}$, we claim that the
combination of $\Scon_{O}$ and $\Scon_{\hat{O}}$ still provides a good
overview of the factorization shortcomings.

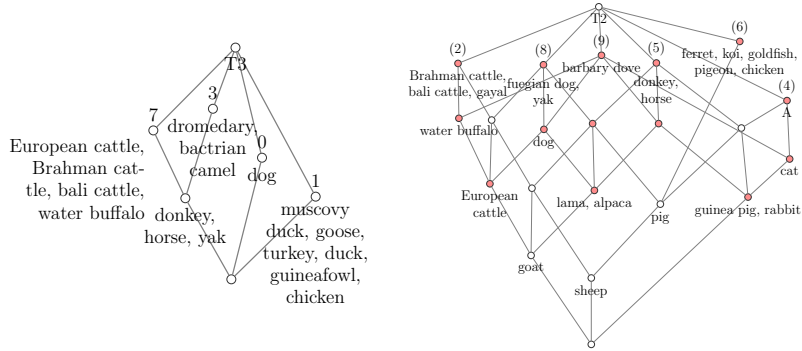
\begin{figure}[t]
  \begin{minipage}{.5\linewidth}
\scalebox{0.6}{\colorlet{mivertexcolor}{white}
\colorlet{jivertexcolor}{white}
\colorlet{vertexcolor}{white}
\colorlet{bordercolor}{black!80}
\colorlet{linecolor}{gray}
\tikzset{vertexbase/.style={semithick, shape=circle, inner sep=2pt, outer sep=0pt, draw=bordercolor},%
  vertex/.style={vertexbase, fill=vertexcolor!45},%
  mivertex/.style={vertexbase, fill=mivertexcolor!45},%
  jivertex/.style={vertexbase, fill=jivertexcolor!45},%
  divertex/.style={vertexbase, top color=mivertexcolor!45, bottom color=jivertexcolor!45},%
  conn/.style={-, thick, color=linecolor}%
}
\tikzstyle{n2} = [text width=4cm,align=center,label distance = 0.3cm]
\tikzstyle{n} = [text width=2.5cm,align=center]
\tikzstyle{d} = [text width=3cm,align=right]
\begin{tikzpicture}[scale=0.6, label distance =0.1cm, font=\large]
  \begin{scope} 
    \begin{scope} 
      \foreach \nodename/\nodetype/\xpos/\ypos in {%
        0/vertex/-0.11926653970573042/1.4969629254185257,
        1/vertex/-1.8178261914700151/4.49303342228053,
        2/vertex/2.9947594885287945/4.55201118796679,
        3/vertex/1.0131065614704617/5.991068670711532,
        4/vertex/-3.0/7.0,
        5/vertex/-0.8/7.8,
        6/vertex/0.03407565107854538/10.060534503063467
      } \node[\nodetype] (\nodename) at (\xpos, \ypos) {};
    \end{scope}
    \begin{scope} 
      \path (1) edge[conn] (5);
      \path (2) edge[conn] (6);
      \path (0) edge[conn] (3);
      \path (0) edge[conn] (1);
      \path (1) edge[conn] (4);
      \path (0) edge[conn] (2);
      \path (4) edge[conn] (6);
      \path (3) edge[conn] (6);
      \path (5) edge[conn] (6);
    \end{scope}
    \begin{scope} 
      \foreach \nodename/\labelpos/\labelopts/\labelcontent in {%
        1/below/n/{donkey, horse, yak},
        2/below/n/{muscovy duck, goose, turkey, duck, guineafowl, chicken},
        2/above/n/{1},
        3/below/n/{dog},
        3/above/n/{0},
        4/below left/d/{European cattle, Brahman cattle, bali cattle, water buffalo},
        4/above/n/{7},
        5/below/n/{dromedary, bactrian camel},
        5/above/n/{3},
        6/below/n/{T3}
      } \coordinate[label={[\labelopts]\labelpos:{\labelcontent}}](c) at (\nodename);
    \end{scope}
  \end{scope}
\end{tikzpicture}}    
  \end{minipage}\hspace{-1cm}
  \begin{minipage}{.49\linewidth}
    \scalebox{0.44}{\colorlet{mivertexcolor}{white}
\colorlet{jivertexcolor}{red}
\colorlet{vertexcolor}{mivertexcolor!50}
\colorlet{bordercolor}{black!80}
\colorlet{linecolor}{gray}
\tikzset{vertexbase/.style={semithick, shape=circle, inner sep=2pt, outer sep=0pt, draw=bordercolor},%
  vertex/.style={vertexbase, fill=vertexcolor!45},%
  mivertex/.style={vertexbase, fill=mivertexcolor!45},%
  jivertex/.style={vertexbase, fill=jivertexcolor!45},%
  divertex/.style={vertexbase, top color=mivertexcolor!45, bottom color=jivertexcolor!45},%
  conn/.style={-, thick, color=linecolor}%
}
\tikzstyle{n2} = [text width=4cm,align=center,label distance=0.3cm]
\tikzstyle{l} = [text width=4cm,align=right]
\tikzstyle{n} = [text width=4cm,align=center]
\tikzstyle{r} = [text width=4cm,align=left]
\begin{tikzpicture}[scale=0.3, label distance=0.1cm, font=\large]
  \begin{scope} 
    \begin{scope} 
      \foreach \nodename/\nodetype/\xpos/\ypos in {%
        0/vertex/0.6912162162162261/4.94594594594605,
        1/vertex/0.6912162162162261/11.635135135135243,
        2/vertex/-5.37364864864864/13.909459459459569,
        3/vertex/7.647972972972987/19.082432432432544,
        4/jivertex/16.47770270270272/19.79594594594606,
        5/jivertex/1.003378378378386/20.46486486486498,
        6/vertex/-5.284459459459455/20.68783783783795,
        7/jivertex/-9.52094594594594/21.133783783783898,
        8/jivertex/20.669594594594614/23.631081081081195,
        9/jivertex/-4.080405405405401/26.618918918919036,
        10/vertex/15.808783783783802/26.75270270270282,
        11/jivertex/0.82500000000001/27.198648648648764,
        12/jivertex/7.469594594594611/27.198648648648764,
        13/vertex/-9.342567567567563/27.555405405405523,
        14/jivertex/-12.642567567567564/27.7337837837839,
        15/jivertex/20.402027027027046/29.517567567567685,
        16/jivertex/15.645270270270288/35.479729729729847,
        17/jivertex/1.7168918918919047/34.08513513513525,
        18/jivertex/-4.124999999999993/33.129729729729846,
        19/jivertex/-12.731756756756752/33.26351351351363,
        20/jivertex/7.246621621621635/33.30810810810823,
        21/vertex/1.449324324324337/38.97162162162174
      } \node[\nodetype] (\nodename) at (\xpos, \ypos) {};
    \end{scope}
    \begin{scope} 
      \path (20) edge[conn] (21);
      \path (5) edge[conn] (11);
      \path (8) edge[conn] (15);
      \path (16) edge[conn] (21);
      \path (14) edge[conn] (17);
      \path (6) edge[conn] (13);
      \path (1) edge[conn] (6);
      \path (11) edge[conn] (20);
      \path (7) edge[conn] (14);
      \path (13) edge[conn] (19);
      \path (1) edge[conn] (3);
      \path (13) edge[conn] (18);
      \path (11) edge[conn] (18);
      \path (9) edge[conn] (18);
      \path (6) edge[conn] (11);
      \path (14) edge[conn] (19);
      \path (10) edge[conn] (15);
      \path (2) edge[conn] (7);
      \path (19) edge[conn] (21);
      \path (4) edge[conn] (10);
      \path (0) edge[conn] (2);
      \path (3) edge[conn] (11);
      \path (7) edge[conn] (13);
      \path (10) edge[conn] (20);
      \path (2) edge[conn] (5);
      \path (4) edge[conn] (12);
      \path (9) edge[conn] (17);
      \path (2) edge[conn] (6);
      \path (17) edge[conn] (21);
      \path (3) edge[conn] (16);
      \path (18) edge[conn] (21);
      \path (0) edge[conn] (1);
      \path (5) edge[conn] (12);
      \path (12) edge[conn] (17);
      \path (8) edge[conn] (17);
      \path (0) edge[conn] (4);
      \path (3) edge[conn] (10);
      \path (12) edge[conn] (20);
      \path (5) edge[conn] (9);
      \path (4) edge[conn] (8);
      \path (15) edge[conn] (21);
      \path (7) edge[conn] (9);
    \end{scope}
    \begin{scope} 
      \foreach \nodename/\labelpos/\labelopts/\labelcontent in {%
        1/below/n/{sheep},
        2/below/n/{goat},
        3/below/n/{pig},
        4/below/n/{guinea pig, rabbit},
        5/below/n/{lama, alpaca},
        7/below/n/{European\\ cattle},
        8/below/n/{cat},
        9/below/n/{dog},
        14/below/n/{water buffalo},
        15/below/n/{A},
        15/above/n/{(4)},
        16/below/n/{ferret, koi, goldfish, pigeon, chicken},
        16/above/n/{(6)},
        17/below/n/{barbary dove},
        17/above/n/{(9)},
        18/below/n2/{fuegian dog,\\ yak},
        18/above/n/{(8)},
        19/below/n/{Brahman cattle,\\ bali cattle, gayal},
        19/above/n/{(2)},
        20/below/n2/{donkey,\\ horse},
        20/above/n/{(5)},
        21/below/n/{T2}
      } \coordinate[label={[\labelopts]\labelpos:{\labelcontent}}](c) at (\nodename);
    \end{scope}
  \end{scope}
\end{tikzpicture}}
  \end{minipage}  
  \caption{The concept lattice of all valid (left) and invalid (right)
    attributes of the \emph{Domestic} scale-measure.  Extents in the
    lattice drawing of the invalid attributes that are not extents in
    the \emph{Domesticated Animals} context are marked in red.
    {A=$\{$society finch, silkmoth, fancy mouse, mink, fancy rat,
      striped skunk, Guppy, canary, silver fox$\}$ T3= $\{$fuegian
      dog, lama, sheep, ferret, pig, alpaca, society finch, goat,
      silkmoth, fancy mouse, koi, guinea pig, rabbit, hedgehog, mink,
      fancy rat, striped skunk, goldfish, barbary dove, Guppy, canary,
      pigeon, silver fox, cat, gayal$\}$ T2=$\{$dromedary, muscovy
      duck, bactrian camel, goose, turkey, hedgehog, duck,
      guineafowl$\}$}}
  \label{fig:err3}
\end{figure}

\section{Related Work}
To cope with large data sets, a multitude of methods was introduced to
reduce the dimensionality. One such method is the factorization of a
data set $\context$ into two factors (scales) $\Scon,H$ whose product
$\Scon \cdot H$ estimates $\context$. For binary data sets, a related
problem is known as the \emph{discrete basis
problem}~\cite{descreteBasis}: For a given $k\in\mathbb{N}$, compute
two binary factors $\Scon \in \mathbb{B}^{n\times k}$ and $H\in
\mathbb{B}^{k\times m}$ for which $\|K-SH\|$ is minimal. This problem
is known to be NP-hard. This hardness result lead to the development
of several approximation algorithms \cite{bmf, FCABMF}. For example,
one approach uses formal concepts as attributes of $\Scon$
\cite{FCABMF} and objects of $H$. It is shown that a solution to the
discrete basis problem can be given in terms of this
representation. However, since the initial problem is still NP-hard,
this approach is computationally intractable for large data sets.

The BMF approach~\cite{bmf} adapts a non-negative matrix factorization
by using a penalty and thresholding function. This algorithm optimizes
two initial matrices to minimize the error $\|K-SH\|$. This procedure
does compute an approximation to the solution of the discrete basis
problem and has lower computational cost. An additional drawback of
the BMF algorithm is that it may introduce closed objects sets in
$\Scon\cdot H$ that are not closed in the data set $K$. The main idea
of the conceptual scaling error is the efficient computation and
quantification of said closed sets.

Other evaluations of BMF, besides $\|K-S\cdot H\|$, have
been considered~\cite{NMFApplication2}. They investigate the quality of
implications in $S\cdot H$ for some classification task. Additionally, they
use different measures~\cite{MeasuresNMF}, e.g., \emph{fidelity}
and \emph{descriptive loss}. Other statistical approaches often find
\emph{euclidean loss}, \emph{Kullback-Leibler} divergence,
\emph{Residual Sum of Squares}, adequate.
%
All previously mentioned evaluation criteria do not account for the
complete conceptual structure of the resulting data set. Moreover,
they are not intrinsically able to pinpoint to the main error portions
of the resulting scale data sets. Furthermore, approaches based on the
computation of implications are infeasible for larger data set. An
advantage of our approach is the polynomial estimation of the
conceptual error in the size of $\Scon$ and $\context$ through the
attribute scaling error.

\section{Conclusion}
With our work we have presented a new approach to evaluate dimension
reduction methods in particular and data scaling methods in general.
The proposed conceptual scaling error was derived from a natural
notion of continuity, as used in scale-measures within the realm of
FCA. Beyond the quantification of the conceptual error, we have
succeeded in presenting a method to explicitly represent the error
generated by the dimension reduction, and to visualize it with the
help of conceptual lattices. For large data set we demonstrated a
method for estimating the scaling error in time polynomial with
respect to the data set size. In our experiments we showed that even
though a factorization using BMF terminates with apparently high
accuracy, the calculated scale does reflect only about 55\%
consistent extents. To prevent such high conceptual errors, we
envision an adaption of BMF that optimizes additionally for low
conceptual errors.

\printbibliography
\end{document}